\documentclass[journal]{IEEEtran}

\usepackage{amsmath}
\usepackage{amssymb}
\usepackage{amsthm}
\usepackage{bm}
\usepackage{cite}
\usepackage[inline]{enumitem}
\usepackage{flushend}
\usepackage{graphicx}
\usepackage{hyperref}
\usepackage{mathrsfs}
\usepackage{mathtools}
\usepackage{xcolor}
\definecolor{myred}{rgb}{1,0,0.25}

\usepackage{macros}

\usepackage[capitalise, nameinlink]{cleveref}
\hypersetup{bookmarksopen=true}

\usepackage{crossreftools}
\crefformat{equation}{\textup{#2(#1)#3}}
\crefrangeformat{equation}{\textup{#3(#1)#4--#5(#2)#6}}
\crefmultiformat{equation}{\textup{#2(#1)#3}}{ and \textup{#2(#1)#3}}
{, \textup{#2(#1)#3}}{, and \textup{#2(#1)#3}}
\crefrangemultiformat{equation}{\textup{#3(#1)#4--#5(#2)#6}}%
{ and \textup{#3(#1)#4--#5(#2)#6}}{, \textup{#3(#1)#4--#5(#2)#6}}{, and \textup{#3(#1)#4--#5(#2)#6}}

\Crefformat{equation}{#2Equation~\textup{(#1)}#3}
\Crefrangeformat{equation}{Equations~\textup{#3(#1)#4--#5(#2)#6}}
\Crefmultiformat{equation}{Equations~\textup{#2(#1)#3}}{ and \textup{#2(#1)#3}}
{, \textup{#2(#1)#3}}{, and \textup{#2(#1)#3}}
\Crefrangemultiformat{equation}{Equations~\textup{#3(#1)#4--#5(#2)#6}}%
{ and \textup{#3(#1)#4--#5(#2)#6}}{, \textup{#3(#1)#4--#5(#2)#6}}{, and \textup{#3(#1)#4--#5(#2)#6}}

\crefname{rem}{Remark}{Remarks}
\crefname{lemma}{Lemma}{Lemmas}
\crefname{cor}{Corollary}{Corollaries}
\crefname{fig}{Figure}{Figures}
\crefname{prop}{Proposition}{Propositions}
\crefname{defn}{Definition}{Definitions}
\crefname{app}{Appendix}{Appendices}

\newtheorem{theorem}{Theorem}
\newtheorem{lemma}[theorem]{Lemma}
\newtheorem{corollary}[theorem]{Corollary}
\newtheorem{proposition}[theorem]{Proposition}

\theoremstyle{definition}
\newtheorem{definition}[theorem]{Definition}
\newtheorem{remark}[theorem]{Remark}

\begin{document}
\bstctlcite{IEEEexample:BSTcontrol}

\title{Near-Minimax Optimal Estimation With \\ Shallow ReLU Neural Networks}

\author{Rahul~Parhi,~\IEEEmembership{Member,~IEEE},
        and~Robert~D.~Nowak,~\IEEEmembership{Fellow,~IEEE}%
        \thanks{This research was partially supported by NSF grant DMS-2134140, ONR MURI grant N00014-20-1-2787, AFOSR/AFRL grant FA9550-18-1-0166, 
        and the NSF  Graduate  Research  Fellowship  Program  under  grant  DGE-1747503.}
        \thanks{Rahul Parhi was with the Department of Electrical and Computer Engineering, University of Wisconsin--Madison,  Madison, WI, USA. He is now with the Biomedical Imaging Group, \'Ecole Polytechnique F\'ed\'erale de Lausanne, Lausanne, Switzerland (e-mail: rahul.parhi@epfl.ch).}
        
        \thanks{Robert D. Nowak is with the Department of Electrical and Computer Engineering, University of Wisconsin--Madison,  Madison, WI, USA (e-mail: rdnowak@wisc.edu).}}

\maketitle

\begin{abstract}
    We study the problem of estimating an unknown function from noisy data using shallow ReLU neural networks. The estimators we study minimize the sum of squared data-fitting errors plus a regularization term proportional to the squared Euclidean norm of the network weights. This minimization corresponds to the common approach of training a neural network with weight decay. We quantify the performance (mean-squared error) of these neural network estimators when the data-generating function belongs to the second-order Radon-domain bounded variation space. This space of functions was recently proposed as the natural function space associated with shallow ReLU neural networks. We derive a minimax lower bound for the estimation problem for this function space and show that the neural network estimators are minimax optimal up to logarithmic factors. This minimax rate is immune to the curse of dimensionality. We quantify an explicit gap between neural networks and linear methods (which include kernel methods) by deriving a linear minimax lower bound for the estimation problem, showing that linear methods necessarily suffer the curse of dimensionality in this function space. As a result, this paper sheds light on the phenomenon that neural networks seem to break the curse of dimensionality.

\end{abstract}


\section{Introduction}
The fundamental building blocks of neural networks are \emph{ridge functions}. A ridge function is a multivariate function mapping $\R^d \to \R$ of the form
\[
    \vec{x} \mapsto \rho(\vec{w}^\T\vec{x}), \quad \vec{x} \in \R^d,
\]
where $\rho: \R \to \R$ is referred to as the \emph{profile} of the ridge function and $\vec{w} \in \R^d \setminus \curly{\vec{0}}$ is referred to as the \emph{direction} of the ridge function.

This paper studies the problem estimating functions from noisy samples using shallow neural networks, which are superpositions of ridge functions, of the form
\begin{equation}
    f(\vec{x}) = \sum_{k=1}^K v_k \, \rho(\vec{w}_k^\T\vec{x} - b_k), \quad \vec{x} \in \R^d,
    \label{eq:shallow-nn}
\end{equation}
where the $\rho: \R \to \R$ is the \emph{activation function}, $K$ is the \emph{width} of the neural network, and, for $k = 1, \ldots, K$, $v_k \in \R \setminus \curly{0}$ and $\vec{w}_k \in \R^d \setminus \curly{\vec{0}}$ are  the \emph{weights} of the neural network and $b_k \in \R$ are the \emph{biases} of the neural network. Throughout the paper, we will focus on the rectified linear unit (ReLU) activation function, $\rho(x) = \max\curly{0, x}$, which is widely used in practice~\cite{deep-learning}.

We consider the problem of nonparametric function estimation where the goal is to estimate an unknown function $f: \Omega \to \R$, where $\Omega \subset \R^d$ is a bounded domain, from the noisy samples
\begin{equation}
    y_n = f(\vec{x}_n) + \varepsilon_n, \: n = 1, \ldots, N,
    \label{eq:data}
\end{equation}
where the noise $\curly{\varepsilon_n}_{n=1}^N$ are i.i.d. Gaussian random variables and $\curly{\vec{x}_n}_{n=1}^N \subset \Omega$ are the design points. We study the performance of neural network estimators of the form in \cref{eq:shallow-nn} that minimize the objective of the sum of squared data-fitting errors plus a regularization term proportional to the squared Euclidean norm of the network weights.  This minimization corresponds to the common approach of gradient-based training of a neural network with \emph{weight decay}~\cite{weight-decay}.  That is, training a neural network using gradient descent with weight decay is simply gradient descent applied to this objective.
    
In order to quantify the performance of such estimators, we consider cases in which $f$ is an unknown function within a known function space. To this end, we will consider functions mapping $\Omega \to \R$ which belong to the Banach space of functions of second-order bounded variation in the Radon domain, denoted $\RBV^2(\Omega)$. Our recent work in~\cite{ridge-splines,deep-ridge-splines} proposed this Banach space as the ``natural" function space associated with shallow ReLU networks. This space contains several classical multivariate function spaces including certain Sobolev spaces as well as certain \emph{spectral Barron spaces}, pioneered in the seminal work of Barron on approximation and estimation using shallow sigmoidal networks~\cite{uat4}.  

It was first observed in~\cite{uat4} that neural network estimators can be \emph{immune to the curse of dimensionality}. This paper sheds light on this phenomenon.  $\RBV^2(\Omega)$ contains classical multivariate function spaces including the $L^1$- and $L^2$-Sobolev spaces of order $d + 1$, where $d$ is the ambient dimension of the domain $\Omega \subset \R^d$. It is classically known that this sort of Sobolev-regularity is sufficient to overcome the curse of dimensionality. On the other hand, $\RBV^2(\Omega)$ also contains functions that are much less regular. In particular, functions with significant variation and irregularity, but only in a few directions, also belong to $\RBV^2(\Omega)$. For example, any ridge function with a profile that has just its first two weak derivatives in $L^2(\Omega)$ is included in $\RBV^2(\Omega)$.  This shows that $\RBV^2(\Omega)$ may be regarded as a \emph{mixed variation} space~\cite{donoho2000high}, since it contains functions that are more regular in some directions and less in others.  This makes $\RBV^2(\Omega)$ a compelling framework for high-dimensional estimation.  Moreover, the neural network estimators we study are \emph{locally adaptive} to such mixed variation.

Our past work~\cite{ridge-splines,deep-ridge-splines} derives a \emph{neural network representer theorem} which proves that shallow ReLU networks are solutions to data-fitting problems in $\RBV^2(\R^d)$, the space of functions defined on $\R^d$ of second-order bounded variation in the Radon domain. Remarkably, this variational problem can be recast as a finite-dimensional neural network training problem where the regularization corresponds to training a shallow ReLU network with weight decay. This is the reason we view these spaces as the natural function space of shallow ReLU networks. This connection is reminiscent of the classical reproducing kernel Hilbert space (RKHS) representer theorem which says that kernel machines are solutions to data-fitting variational problems over the associated RKHS, although the neural network variational problem is posed over a (non-Hilbertian) Banach space.

We summarize the contributions of this paper below.
\begin{enumerate}
    \item We first discuss how to define $\RBV^2(\Omega)$, where $\Omega \subset \R^d$ is a \emph{bounded domain}, while preserving a representer theorem for shallow ReLU networks.  This implies that data-fitting with functions in $\RBV^2(\Omega)$ can be  recast as a finite-dimensional neural network training problem that may be solved using gradient-descent with weight decay. This result sets the stage for discussing approximation and estimation error for functions in $\RBV^2(\Omega)$.
    
    \item We relate $\RBV^2(\Omega)$ spaces to previously studied function spaces related to shallow neural networks. In particular, we show that $\RBV^2(\Omega)$ is exactly the same (in the sense of equivalent Banach spaces) as the so-called \emph{variation space} associated to shallow ReLU networks that has been studied by a number of authors~\cite{kurkova2001bounds,mhaskar2004tractability,convex-nn,siegel2021characterization}. This provides a novel analytic characterization of this space. Using this characterization, we can apply previously derived optimal approximation rates for functions from the variation space~\cite{convex-nn,siegel} to characterize the optimal approximation rates for functions in $\RBV^2(\Omega)$. The approximation rate (with respect to the  $L^\infty(\Omega)$-norm) is $K^{-\frac{d + 3}{2d}}$, where $K$ is the number of neurons in the approximant. Remarkably, this rate is \emph{immune to the curse of dimensionality}, as it tends to $K^{-1/2}$ as $d \to \infty$. We also show that $\RBV^2(\Omega)$ is \emph{larger} than the second-order spectral Barron space. \label{item:contributions-approx}

    \item We show that a shallow ReLU network that minimizes the sum of squared data-fitting errors plus a regularization term proportional to the sum of squared weights (i.e., training a shallow ReLU network with weight decay to a global minimizer) is a minimax optimal (up to logarithmic factors) estimator when the data are generated according to \cref{eq:data}, where $f \in \RBV^2(\Omega)$. The minimax rate of the mean-squared error is, up to logarithmic factors, $N^{-\frac{d+3}{2d+3}}$. Remarkably, this rate is \emph{immune to the curse of dimensionality}, as it tends to $N^{-1/2}$ as $d \to \infty$.
    
    \item Using the results of this paper, we show that there is a fundamental gap between neural networks and more classical linear methods (which include kernel methods). In particular, we use ridgelet analysis to derive a minimax lower bound for the estimation problem when restricted to linear estimators. We find that the linear minimax lower bound is $N^{-\frac{3}{d + 3}}$, which suffers the curse of dimensionality as $d \to \infty$. This result says that linear methods are suboptimal for estimating functions in $\RBV^2(\Omega)$. We also show this gap qualitatively via numerical experiments.
\end{enumerate}

\subsection{Related Work}
There is a large body of work regarding the problem of statistical estimation with ridge functions, under many different names, including projection pursuit regression~\cite{friedman1981projection}, ridgelet shrinkage~\cite{candes2003ridgelets}, and, of course, estimation with neural networks~\cite{uat4}. The last few years have led to a number of related works that consider the problem of minimax estimation with neural networks~\cite{klusowski2017minimax,imaizumi2019deep,suzuki2018adaptivity,schmidt2020nonparametric,hayakawa2020minimax}. These works fall into two categories:
\begin{enumerate*}[label=\arabic*)]
    \item they consider the problem of estimating a function that is \emph{explicitly synthesized} from a dictionary of neurons;
    
    \item they consider the problem of estimating a function from a particular (classical) space of functions (e.g., H\"older, Sobolev, Besov, etc.)
\end{enumerate*}.
Moreover, the procedures for actually constructing the estimators in these works usually involve greedy algorithms and do not correspond to how neural networks are actually trained in practice. The work of this paper is different from these past works in that we consider the problem of estimating functions from a new, not classical, function space, $\RBV^2(\Omega)$, and study the performance of estimators that correspond to solutions to problem of training shallow ReLU networks with weight decay, a common regularization scheme used when training neural networks in practice.

\subsection{Roadmap}
In \cref{sec:prelim} we introduce notation used in the remainder of the paper. In \cref{sec:rep-thm-Rd} we introduce relevant results from our previous work~\cite{ridge-splines,deep-ridge-splines}. In \cref{sec:bounded-domain} we discuss how to define $\RBV^2(\Omega)$ where $\Omega \subset \R^d$ is a bounded domain and derive a new representer theorem for shallow ReLU networks by considering variational problems over $\RBV^2(\Omega)$.  In \cref{sec:other-spaces} we relate $\RBV^2(\Omega)$ to previously studied function spaces associated to shallow networks. In \cref{sec:approximation} we derive optimal approximation rates for functions in $\RBV^2(\Omega)$, where the approximants are shallow ReLU networks. In \cref{sec:estimation} we show that shallow ReLU network estimators are minimax optimal (up to logarithmic factors) for estimating functions in $\RBV^2(\Omega)$. In \cref{sec:nn-not-kernel} we show that there is a fundamental gap between neural networks and linear methods (including kernel methods).

\section{Preliminaries \& Notation} \label{sec:prelim}
Let $L^p(\Omega)$ denote the usual Lebesgue space, where $\Omega$ is a domain (either bounded or unbounded). This space is a Banach space when equipped with the norm
\begin{align*}
    \norm{f}_{L^p(\Omega)} &\coloneqq \paren{\int_\Omega \abs{f(\vec{x})}^p \dd \vec{x}}^{1/p}, \quad 1 \leq p < \infty, \\
    \norm{f}_{L^\infty(\Omega)} &\coloneqq \esssup_{\vec{x} \in \Omega} \, \abs{f(\vec{x})}, \quad p = \infty.
\end{align*}
When we do not specify the underlying measure, it will correspond to the Haar measure of $\Omega$ (e.g., Lebesgue measure when $\Omega = \R^d$ or the surface measure when $\Omega = \Sph^{d-1}$, the surface of the Euclidean sphere in $\R^d$). When we do specify a particular measure, say $\mu$, we will write $L^p(\Omega; \mu)$.

We will also work with the Banach space of finite Radon measures on $\Omega$, denoted $\M(\Omega)$. The norm $\norm{\dummy}_{\M(\Omega)}$ is exactly the \emph{total variation norm} (in the sense of measures). We can view this space as a subspace of distributions (generalized functions) on $\Omega$. The space $\M(\Omega)$ may be regarded as a ``generalization'' of $L^1(\Omega)$ in the sense that if $f \in L^1(\Omega)$, $\norm{f}_{L^1(\Omega)} = \norm{f}_{\M(\Omega)}$, but $\M(\Omega)$ is a strictly larger space that also contains the shifted Dirac impulses $\delta(\dummy - \vec{x}_0)$, $\vec{x}_0 \in \Omega$, such that $\norm{\delta(\dummy - \vec{x}_0)}_{\M(\Omega)} = 1$. We also remark that the $\M$-norm is the continuous-domain analogue of the $\ell^1$-norm. We refer the reader to~\cite[Chapter~7]{folland} for more details about this space.

We will also use the notation $a_N \lesssim b_N$ to mean there exists a constant $C$ (independent of $N$) such that $a_N \leq C\, b_N$, $a_N \gtrsim b_N$ to mean $b_N \lesssim a_N$, and $a_N \asymp b_N$ to mean $a_N \lesssim b_N$ and $a_N \gtrsim b_N$. We will also subscript $\lesssim$, $\gtrsim$, and $\asymp$ with any parameters that the implicit constant depends on.
\section{Shallow Neural Networks, Splines, and Variational Methods} \label{sec:rep-thm-Rd}
In this section we will discuss relevant results from our prior work in~\cite{ridge-splines,deep-ridge-splines}, making connections between shallow neural networks, splines, and variational methods. Our work in~\cite{ridge-splines} proved a \emph{representer theorem} for single-hidden layer ReLU networks with scalar outputs by considering variational problems over the space of functions of second-order bounded variation in the Radon domain. The Radon transform of a function $f: \R^d \to \R$ is given by
\[
    \Radon{f}(\vec{\gamma}, t) \coloneqq \int_{\curly{\vec{x}: \vec{\gamma}^\T
  \vec{x} = t}} f(\vec{x}) \dd s(\vec{x}), \quad (\vec{\gamma},t) \in \cyl,
\]
where $s$ denotes the $(d-1)$-dimensional Lebesgue measure on the hyperplane $\curly{\vec{x}
\st \vec{\gamma}^\T \vec{x} = t}$. The Radon domain is parameterized by a \emph{direction} $\vec{\gamma} \in \Sph^{d-1}$ and an \emph{offset} $t \in \R$. When working with the Radon transform of functions defined on $\R^d$, the following \emph{ramp filter} arises in the Radon inversion formula
\[
    \Lambda^{d-1} = (-\partial_t^2)^{\frac{d-1}{2}},
\]
where $\partial_t$ denotes the partial derivative with respect to the offset variable, $t$, of the Radon domain and fractional powers are defined in terms of Riesz potentials. The space of functions of second-order bounded variation in the Radon domain is then given by
\begin{equation}
    \RBV^2(\R^d) = \curly{f \in L^{\infty, 1}(\R^d) \st \RTV^2(f) < \infty},
    \label{eq:RBV}
\end{equation}
where $L^{\infty, 1}(\R^d)$ is the Banach space\footnote{It is a Banach space when equipped with the norm $\norm{f}_{\infty, 1} \coloneqq \esssup_{\vec{x} \in \R^d} \abs{f(\vec{x})}(1 + \norm{\vec{x}}_2)^{-1}$.} of functions mapping $\R^d \to \R$ of at most linear growth and
\begin{equation}
    \RTV^2(f) = c_d \norm*{\partial_t^2 \ramp^{d-1} \RadonOp f}_{\M(\cyl)}
    \label{eq:RTV}
\end{equation}
denotes the second-order total variation of a function in the offset variable of the (filtered) Radon domain, where $c_d^{-1} = 2(2\pi)^{d-1}$ is a dimension-dependant constant that arises when working with the Radon transform. Note that all the operators that appear in \cref{eq:RTV} must be understood in the distributional sense. We refer the reader to~\cite[Section~3]{ridge-splines} for more details.

The $\RTV^2$-seminorm was first proposed in~\cite{function-space-relu} and studied in extensive detail in~\cite{ridge-splines,deep-ridge-splines}. When equipped with the norm
\[
    \norm{f}_{\RBV^2(\R^d)} \coloneqq \RTV^2(f) + \abs{f(\vec{0})} + \sum_{k=1}^d \abs{f(\vec{e}_k) - f(\vec{0})},
\]
where $\curly{\vec{e}_k}_{k=1}^d$ denotes the canonical basis of $\R^d$, $\RBV^2(\R^d)$ is a Banach space~\cite[Lemma~2.4]{deep-ridge-splines}. In particular, it is a Banach space with a sparsity-promoting norm as $\RTV^2(\dummy)$ is defined via an $\M$-norm.  The terms $\abs{f(\vec{0})} + \sum_{k=1}^d \abs{f(\vec{e}_k) - f(\vec{0})}$ that appear in the above display impose a norm on the null space of $\RTV^2(\dummy)$, which corresponds to affine functions on $\R^d$, and is  an upper bound on the Lipschitz constant of the affine portion of $f$.

Intuitively, the $\RTV^2$-seminorm measures sparsity of second derivatives in the Radon domain. The Radon transform naturally arises when working with ridge functions. In particular, the second derivative of the (filtered) Radon transform of a ReLU ridge function is essentially a Dirac impulse located at the weight and bias of the ReLU ridge function~\cite[Lemma~17]{ridge-splines}. This arises due to the fact that in the univariate case, the second derivative of the ReLU is a Dirac impulse. Thus, the seminorm in \cref{eq:RTV} favors ReLU ridge functions and so functions in $\RBV^2(\R^d)$ with small $\RTV^2$-seminorm will typically take the form of a sparse superposition of ReLU ridge functions. We now state the main result of~\cite{ridge-splines}.

\begin{proposition}[{special case of~\cite[Theorem~1]{ridge-splines}}] \label[prop]{prop:rep-thm}
    Let $\ell(\dummy, \dummy): \R \times \R \to \R$ be a strictly convex, coercive, and lower-semicontinuous in its second argument loss function and let $\lambda > 0$ be an adjustable regularization parameter. Then, for any data $\curly{(\vec{x}_n, y_n)}_{n=1}^N \subset \R^d \times \R$, there exists a solution to the variational problem
    \begin{equation}
        \min_{f \in \RBV^2(\R^d)} \: \sum_{n=1}^N \ell(y_n, f(\vec{x}_n)) + \lambda \, \RTV^2(f)
        \label{eq:variational-problem}
    \end{equation}
    that takes the form of a shallow ReLU network plus an affine function. In particular, it takes the form
    \begin{equation}
        s(\vec{x}) = \sum_{k=1}^K v_k \, \rho(\vec{w}_k^\T \vec{x} - b_k) + \vec{c}^\T\vec{x} + c_0, \quad \vec{x} \in \R^d,
        \label{eq:ridge-spline}
    \end{equation}
    where $K \leq N - (d + 1)$, $\rho$ is the ReLU, $\vec{w}_k \in \Sph^{d-1}$, $v_k \in \R \setminus \curly{0}$, $b_k \in \R$, $\vec{c} \in \R^d$, and $c_0 \in \R$.
\end{proposition}
We remark that the affine function that appears in \cref{eq:ridge-spline} is known as a \emph{skip connection} in neural network parlance~\cite{skip-connections}. In other words, \cref{eq:ridge-spline} is a shallow ReLU network with a skip connection.

\subsection{Shallow Neural Networks and Splines} \label{sec:shallow-nn-splines}

When $d = 1$, the space $\RBV^2(\R^d)$ is the classical second-order bounded variation space
\[
    \BV^2(\R) \coloneqq \curly{f: \R \to \R \st \TV^2(f) < \infty},
\]
where
\[
    \TV^2(f) \coloneqq \norm*{\D^2 f}_{\M(\R)}
\]
is the second-order total variation of a function $f: \R \to \R$, where $\D$ is the (distributional) derivative operator~\cite[Section~5.1]{ridge-splines}. In this case, the result of \cref{prop:rep-thm} recovers the classical representer theorem for locally adaptive linear splines, which dates back to the 1970s~\cite{fisher-jerome,locally-adaptive-regression-splines,L-splines}. Moreover, we also have that $\RTV^2(f) = \TV^2(f)$~\cite[Section~5.1]{ridge-splines}.

\subsection{Connections to Neural Network Training} \label{sec:nn-training}
We view $\RBV^2(\R^d)$ as the natural function space associated with shallow ReLU networks since the problem in \cref{eq:variational-problem} can be recast as a finite-dimensional neural network training problem that corresponds to training a sufficiently wide shallow ReLU network (with a skip connection) with weight decay or with the so-called ``path-norm'' regularizer. In particular, consider the shallow ReLU network with a skip connection:
\[
    f_\vec{\theta}(\vec{x}) = \sum_{k=1}^K v_k \, \rho(\vec{w}_k^\T \vec{x} - b_k) + \vec{c}^\T\vec{x} + c_0,
\]
where $\vec{\theta}$ denotes the parameters of the neural network, i.e., $\{v_k,\vec{w}_k,b_k\}_{k=1}^K$, $\vec{c}$ and $c_0$. Then, it was shown in~\cite[Theorem~8]{ridge-splines} that, the solutions to either of the following (equivalent) finite-dimensional neural network training problems
\begin{align}
    &\min_{\vec{\theta} \in \Theta} \sum_{n=1}^N \ell(y_n, f_\vec{\theta}(\vec{x}_n)) + \frac{\lambda}{2} \sum_{k=1}^K \, \abs{v_k}^2 + \norm{\vec{w}_k}_2^2  \label{eq:weight-decay} \\
    &\min_{\vec{\theta} \in \Theta} \sum_{n=1}^N \ell(y_n, f_\vec{\theta}(\vec{x}_n)) + \lambda \sum_{k=1}^K \abs{v_k}\norm{\vec{w}_k}_2 \label{eq:path-norm}
\end{align}
where $\Theta = \R^M$ is the parameter space and $M$ is the total number of scalar parameters of network, are solutions to the variational problem in \cref{eq:variational-problem}, so long as $K \geq N - (d + 1)$. The problem in \cref{eq:weight-decay} corresponds to training a shallow ReLU network with weight decay~\cite{weight-decay} and the problem in \cref{eq:path-norm} corresponds to training a neural network with path-norm regularization~\cite{path-norm}. Therefore, the above says that trained\footnote{Assuming that the network is trained to a global minimizer.} shallow ReLU networks are ``optimal'' with respect to the space $\RBV^2(\R^d)$. This result follows from the fact that
\begin{equation}
    \RTV^2(f_\vec{\theta}) = \sum_{k=1}^K \abs{v_k}\norm{\vec{w}_k}_2,
    \label{eq:nn-norm}
\end{equation}
which can be viewed as a kind of $\ell^1$-norm, giving insight into the sparsity promoting nature of the $\RTV^2$-seminorm on neural network parameters\footnote{The equality in \cref{eq:nn-norm} assumes that the neural network is written in reduced form, i.e., the weight bias pairs $(\vec{w}_k, b_k)$ $k = 1, \ldots, K$ are unique up to certain symmetries. See~\cite{ridge-splines} for more details.}. Moreover, this result also gives insight into the sparsity-promoting nature of training a shallow ReLU network with weight decay. We refer the reader to~\cite{ridge-splines} for more details about recasting the problem in \cref{eq:variational-problem} as the problems in \cref{eq:weight-decay,eq:path-norm}, the equivalence of \cref{eq:weight-decay,eq:path-norm}, and the derivation of the equality in \cref{eq:nn-norm}.

This result also says, in the univariate case, that the function learned by training a sufficiently wide ReLU network with weight decay or with path-norm regularization on data is a locally adaptive linear spline~\cite{relu-linear-spline,min-norm-nn-splines}.

\section{The \texorpdfstring{$\RBV^2$}{RBV2}-Space on a Bounded Domain} \label{sec:bounded-domain}

In approximation theory and nonparametric function estimation it is common to quantify error with respect to the $L^p(\Omega)$-norm, $1 \leq p \leq \infty$, where $\Omega \subset \R^d$ is a bounded domain. Therefore, we are interested in working with the $\RBV^2$-space defined on a \emph{bounded domain}. In this section we will define the $\RBV^2$-space on a bounded domain while still maintaining a similar representer theorem as in $\RBV^2(\R^d)$.

We can define the $\RBV^2$-space on a bounded domain $\Omega \subset \R^d$ using the standard approach of considering restrictions of functions in $\RBV^2(\R^d)$. This provides the following definition:
\[
    \RBV^2(\Omega) \coloneqq \curly{f \in \mathscr{D}'(\Omega) \st \exists g \in \RBV^2(\R^d) \,\subj\, \eval{g}_\Omega = f},
\]
where $\mathscr{D}'(\Omega)$ denotes the space of distributions (generalized functions) on $\Omega$. Similarly, we can define the second-order total variation in the Radon domain of a function $f$ defined on a bounded domain $\Omega \subset \R^d$:
\begin{equation}
    \RTV^2_\Omega(f) \coloneqq \inf_{g \in \RBV^2(\R^d)} \RTV^2(g) \:\:\:\subj\:\:\:\eval{g}_\Omega = f.
    \label{eq:RTV-domain}
\end{equation}
This gives an alternative characterization of $\RBV^2(\Omega)$ as
\[
    \RBV^2(\Omega) = \curly{f \in \mathscr{D}'(\Omega) \st \RTV^2_\Omega(f) < \infty}.
\]
We also remark that since $\RBV^2(\R^d)$ is a Banach space, $\RBV^2(\Omega)$ is also a Banach space. In particular, it is a Banach space when equipped with the norm
\[
    \norm{f}_{\RBV^2(\Omega)} \coloneqq \inf_{g \in \RBV^2(\R^d)} \norm{g}_{\RBV^2(\R^d)} \quad\subj\quad \eval{g}_\Omega = f.
\]

\subsection{Extensions From \texorpdfstring{$\RBV^2(\Omega)$}{RBV2(Omega)} to \texorpdfstring{$\RBV^2(\R^d)$}{RBV2(Rd)}}
In this section we will discuss how to identify functions in $\RBV^2(\Omega)$ with functions in $\RBV^2(\R^d)$, where $\Omega \subset \R^d$ is a bounded domain.

\begin{lemma} \label{lemma:RBV-domain-identification}
Let $\Omega \subset \R^d$ be a bounded domain. Given $f \in \RBV^2(\Omega)$, there exists an extension $f_\mathsf{ext} \in \RBV^2(\R^d)$ that admits an integral representation
\[
    f_\mathsf{ext}(\vec{x}) = \int_\cyl \rho(\vec{w}^\T\vec{x} - b) \dd\mu(\vec{w}, b) + \vec{c}^\T\vec{x} + c_0,
\]
such that $\supp \mu \subset Z_\Omega$, where $Z_\Omega$ is the set
\begin{equation}
    \cl{\curly{\vec{z} = (\vec{w}, b) \in \cyl \st \curly{\vec{x} \st \vec{w}^\T\vec{x} = b} \cap \Omega \neq \varnothing}},
    \label{eq:Z-Omega}
\end{equation}
where $\cl{A}$ denotes the closure of the set $A$. This extension has the property that $\eval{f_\mathsf{ext}}_\Omega = f$ and
\[
    \RTV^2_\Omega(f) = \RTV^2(f_\mathsf{ext}) = \norm{\mu}_{\M(\cyl)} = \norm{\eval{\mu}_{Z_\Omega}}_{\M(Z_\Omega)}.
\]
\end{lemma}
The set $Z_\Omega$ simply excludes ReLU functions that are linear functions (no activation threshold) when restricted to $\Omega$. The proof of \cref{lemma:RBV-domain-identification} relies on several properties of the space $\RBV^2(\R^d)$ from our previous work in~\cite{ridge-splines}. We introduce the relevant background and then prove \cref{lemma:RBV-domain-identification} in \cref{app:extension-restriction}.

\begin{remark} \label{rem:identification}
    When
    \begin{equation}
        \Omega = \B_1^d \coloneqq \curly{\vec{x} \in \R^d \st \norm{\vec{x}}_2 \leq 1},
        \label{eq:Euclidean-ball}
    \end{equation}
    the Euclidean unit ball in $\R^d$, we have that $Z_\Omega$ from \cref{eq:Z-Omega} is exactly
    \[
        Z_\Omega = \Sph^{d-1} \times [-1, 1].
    \]
    Therefore, from \cref{lemma:RBV-domain-identification}, we can identify functions in $f \in \RBV^2(\B_1^d)$ with integral representations of the form
    \[
        f(\vec{x}) = \int_{\Sph^{d-1} \times [-1, 1]} \rho(\vec{w}^\T\vec{x} - b) \dd\mu(\vec{w}, b) + \vec{c}^\T\vec{x} + c_0,
    \]
    where $\vec{x} \in \B_1^d$.
\end{remark}

\begin{remark} \label[rem]{rem:1D-spaces}
    Similar to the discussion in \cref{sec:shallow-nn-splines}, when $d = 1$, the space $\RBV^2(\B_1^d)$ is exactly the classical second-order bounded variation spaces defined on $[-1, 1]$:
    \[
        \BV^2[-1, 1] \coloneqq \curly{f: [-1, 1] \to \R \st \TV^2_{[-1, 1]}(f) < \infty},
    \]
    where
    \[
        \TV^2_{[-1, 1]}(f) \coloneqq \norm*{\D^2 f}_{\M[-1, 1]},
    \]
    where we recall that $\D$ is the (distributional) derivative operator. Moreover, we also have that $\RTV^2_{[-1, 1]}(f) = \TV^2_{[-1, 1]}(f)$.
\end{remark}

\subsection{A Representer Theorem in \texorpdfstring{$\RBV^2(\Omega)$}{RBV2(Omega)}}
We will now discuss a representer theorem for functions in $\RBV^2(\Omega)$, where $\Omega \subset \R^d$ is a bounded domain. For simplicity we will suppose that $\Omega = \B_1^d$ as defined in \cref{eq:Euclidean-ball}. Similar results as those stated in the sequel can be derived for more general bounded domains $\Omega \subset \R^d$. We have the following new representer theorem for data-fitting variational problems over $\RBV^2(\B_1^d)$.
\begin{theorem} \label{thm:rep-thm-domain}
    Let $\ell(\dummy, \dummy): \R \times \R \to \R$ be a strictly convex, coercive, and lower-semicontinuous loss function and let $\lambda > 0$ be an adjustable regularization parameter. Then, for any  data $\curly{(\vec{x}_n, y_n)}_{n=1}^N \subset \B_1^d \times \R$, there exists a solution to the variational problem
    \begin{equation}
        \min_{f \in \RBV^2(\B_1^d)} \: \sum_{n=1}^N \ell(y_n, f(\vec{x}_n)) + \lambda \, \RTV^2_{\B_1^d}(f)
        \label{eq:variational-problem-domain}
    \end{equation}
    that takes the form of a shallow ReLU network with a skip connection. In particular, it takes the form
    \begin{equation}
        s(\vec{x}) = \sum_{k=1}^K v_k \, \rho(\vec{w}_k^\T \vec{x} - b_k) + \vec{c}^\T\vec{x} + c_0, \quad \vec{x} \in \B_1^d,
        \label{eq:ridge-spline-domain}
    \end{equation}
    where $K \leq N - (d + 1)$,
    $\rho$ is the ReLU, $\vec{w}_k \in \Sph^{d-1}$, $v_k \in \R \setminus \curly{0}$, $b_k \in [-1, 1]$, $\vec{c} \in \R^d$ and $c_0 \in \R$.
\end{theorem}

Just as in \cref{sec:nn-training}, we view $\RBV^2(\B_1^d)$ is the natural function space associated with shallow ReLU networks since the problem in \cref{eq:variational-problem-domain} can also be recast as a finite-dimensional neural network training problem that corresponds to training a sufficiently wide shallow ReLU network (with a skip connection) with weight decay or with path-norm regularization as in \cref{eq:weight-decay,eq:path-norm} with the additional restriction that the activation thresholds of the neurons stay within $\B_1^d$. Moreover, similar to \cref{eq:nn-norm} we have in this case that\footnote{Just as in \cref{eq:nn-norm}, the equality in \cref{eq:nn-norm-domain} holds assuming the neural network is written in reduced form.}
\begin{equation}
    \RTV^2_{\B_1^d}(f_\vec{\theta}) = \sum_{k=1}^K \abs{v_k}\norm{\vec{w}_k}_2.
    \label{eq:nn-norm-domain}
\end{equation}

\section{\texorpdfstring{$\RBV^2(\Omega)$}{RBV2(Omega)} and Previously Studied Spaces} \label{sec:other-spaces}
Understanding the properties of shallow neural networks has received much attention since the 1990s starting with the seminal work of Barron~\cite{uat4} in which he studied the approximation properties of shallow sigmoidal networks in the so-called first-order spectral Barron space. The fundamental idea is to consider functions that are \emph{synthesized} from continuously many neurons. Such functions can be expressed as an integral of a neural activation function against a finite (Radon) measure. This idea was adopted by a number of authors in the study of the so-called \emph{variation spaces} of shallow neural networks~\cite{kurkova2001bounds,mhaskar2004tractability,convex-nn,siegel2021characterization}.

In this section we will discuss how $\RBV^2(\Omega)$ is related to previously studied function spaces, including the variation spaces. For simplicity we will suppose that $\Omega = \B_1^d$ as defined in \cref{eq:Euclidean-ball}. Similar results as those stated in the sequel can be derived for more general bounded domains $\Omega \subset \R^d$.

\subsection{Variation Spaces}

Following the setup from~\cite{siegel2021characterization}, in the case of shallow ReLU networks, the associated variation space for functions defined on $\B_1^d$ is defined as
\[
    \begin{aligned}
        \V^2(\B_1^d) \coloneqq &\curly*[\bigg]{f: \B_1^d \to \R \st \\
        &\qquad f = \int_{\Sph^{d-1} \times [-2, 2]} \rho(\vec{w}^\T(\dummy) - b) \dd \mu(\vec{w}, b)},
    \end{aligned}
\]
where $\rho$ is the ReLU and $\mu \in \M(\Sph^{d-1} \times [-2, 2])$. The reason for integrating the $b$ variable over $[-2, 2]$ is so that affine functions can be captured by this space (see~\cite[Section~3]{siegel2021characterization} for more details). This space is known to be a Banach space (see~\cite{siegel2021characterization}) when equipped with the norm
\[
    \begin{aligned}
        \norm{f}_{\V^2(\B_1^d)} \coloneqq &\inf_{\mu \in \M(\Sph^{d-1} \times [-2, 2])} \quad \norm{\mu}_{\M(\Sph^{d-1} \times [-2, 2])} \\
        &\subj \quad f = \int_{\Sph^{d-1} \times [-2, 2]} \rho(\vec{w}^\T(\dummy) - b) \dd \mu(\vec{w}, b).
    \end{aligned}
\]
We will now show that $\RBV^2(\B_1^d)$ and $\V^2(\B_1^d)$ are in fact the same space, providing more evidence that $\RBV^2(\B_1^d)$ is the natural function space associated to shallow ReLU networks.
\begin{theorem} \label{thm:RBV-variation-space}
    $\RBV^2(\B_1^d)$ and $\V^2(\B_1^d)$ are equivalent Banach spaces (i.e., Banach spaces with equivalent norms).
\end{theorem}
\begin{proof}
    Given $f \in \V^2(\B_1^d)$, we have the representation
    \begin{equation}
        f(\vec{x}) = \int_{\Sph^{d-1} \times [-2, 2]} \rho(\vec{w}^\T\vec{x} - b) \dd \mu(\vec{w}, b).
        \label{eq:integral-representation-variation}
    \end{equation}
    Given $g \in \RBV^2(\B_1^d)$, we have from \cref{rem:identification} the representation
    \[
        g(\vec{x}) = \int_{\Sph^{d-1} \times [-1, 1]} \rho(\vec{w}^\T\vec{x} - b) \dd \mu(\vec{w}, b) + \vec{c}^\T\vec{x} + c_0.
    \]
    Clearly we can represent any function in $\V^2(\B_1^d)$ with the representation of $\RBV^2(\B_1^d)$ and vice-versa. Therefore, $\RBV^2(\B_1^d) = \V^2(\B_1^d)$. To see why the norms are equivalent, note that the only difference between the norms is how they handle the null space of the $\RTV^2_{\B_1^d}(\dummy)$ seminorm. Since this null space is the space of affine functions, which is finite-dimensional combined with the fact that all norms are equivalent on finite-dimensional spaces, we have that the norms $\norm{\dummy}_{\RBV^2(\B_1^d)}$ and $\norm{\dummy}_{\V^2(\R^d)}$ are equivalent.
\end{proof}

\subsection{Spectral Barron Spaces}

The spectral Barron spaces were first studied by Barron in~\cite{uat4}. These spaces are defined by
\[
    \mathscr{B}^s(\B_1^d) \coloneqq \curly*[\Bigg]{f: \mathscr{D}'(\B_1^d) \st \inf_{\substack{g \in L^1(\R^d) \\ \eval{g}_{\B_1^d} = f}} \norm*{\hat{\Delta^{s/2} g}}_{L^1(\R^d)} < \infty},
\]
where $\mathscr{D}'(\B_1^d)$ denotes the space of distributions (generalized functions) on $\R^d$, $\hat{\dummy}$ denotes the (generalized) Fourier transform and $\Delta$ denotes the (weak) Laplacian operator where fractional powers are defined in terms of Riesz potentials.

Barron studied the first-order spectral Barron space, $\mathscr{B}^1(\B_1^d)$ in his seminal work about approximation and estimation with shallow sigmoidal networks in~\cite{uat4}. The higher-order variants were studied by a number of authors~\cite{klusowski2018approximation,xu2020finite,siegel2021characterization,spectral-barron-icassp}. In particular, it was shown in~\cite{klusowski2018approximation} that $\mathscr{B}^2(\B_1^d) \subset \V^2(\B_1^d)$. Therefore, by \cref{thm:RBV-variation-space}, we have that $\mathscr{B}^2(\B_1^d) \subset \RBV^2(\B_1^d)$.
    
\subsection{Sobolev Spaces} \label{sec:sobolev}
$\RBV^2(\B_1^d)$ also contains the classical $L^1$- and $L^2$-Sobolev spaces of order $d + 1$. Let $\Omega \subset \R^d$ be a domain (either bounded or unbounded) and recall the the Sobolev space $W^{k, p}(\Omega)$ of functions in $L^p(\Omega)$ with all (weak) derivatives up to and including order $k$ also in $L^p(\Omega)$. This is a Banach space when equipped with the norm
\[
    \norm{f}_{W^{k, p}(\Omega)} \coloneqq \paren{\sum_{\abs{\vec{\alpha}} \leq k} \norm{\partial^{\vec{\alpha}} f}_{L^p(\Omega)}^p}^{1/p},
\]
where $\vec{\alpha} = (\alpha_1, \ldots, \alpha_d) \in \mathbb{N}^d$, $\abs{\vec{\alpha}} = \alpha_1 + \cdots + \alpha_d$, and $\partial^{\vec{\alpha}}$ is the usual multi-index notation for mixed partial derivatives. When $p = 2$, $W^{k, 2}(\Omega)$ is a Hilbert space and we write $H^k(\Omega)$ for $W^{k, 2}(\Omega)$. The following theorem summarizes the relationship between Sobolev spaces and $\RBV^2(\B_1^d)$.

\begin{theorem} \label{thm:sobolev-embedding}
    Given $f \in H^{d+1}(\B_1^d)$,
    \[
        \RTV^2_{\B_1^d}(f) \lesssim_d \norm{f}_{W^{d+1, 1}(\B_1^d)}\lesssim_d \norm{f}_{H^{d+1}(\B_1^d)},
    \]
    where we recall that $\lesssim_d$ means the implicit constant depends on $d$.
    In particular, the above display says that $H^{d+1}(\B_1^d) \subset W^{d+1, 1}(\B_1^d) \subset \RBV^2(\B_1^d)$.
\end{theorem}

The proof of \cref{thm:sobolev-embedding} appears in \cref{app:sobolev-embedding}. We also remark that in order to generalize \cref{thm:sobolev-embedding} to more general bounded domains $\Omega \subset \R^d$ requires that the boundary of $\Omega$ is sufficently nice. It suffices that $\Omega$ has Lipschitz boundary.

\subsection{Observations} \label{sec:RBV-mixed-variation}
The result of \cref{thm:sobolev-embedding} says that very regular functions (those with $d + 1$ derivatives in either $L^1(\B_1^d)$ or $L^2(\B_1^d)$) are contained in $\RBV^2(\B_1^d)$. On the other hand, functions that are not very regular are also in $\RBV^2(\B_1^d)$.
For example, take any univariate function $g \in H^2(\R)$ and use it as the profile of a ridge function
\begin{equation}
    f(\vec{x}) = g(\vec{w}^\T\vec{x}), \quad \vec{x} \in \B_1^d,
    \label{eq:sobolev-ridge}
\end{equation}
where $\vec{w} \in \Sph^{d-1}$. If $g$ has only has two weak derivatives, then the function $f$ is in $\RBV^2(\B_1^d)$ and $H^{2}(\B_1^d)$, but not in $H^{d+1}(\B_1^d)$. Although this function may not be very regular, it only varies in the direction $\vec{w} \in \Sph^{d-1}$. This shows that $\RBV^2(\B_1^d)$ can be viewed as a \emph{mixed variation} space~\cite{donoho2000high} in that it includes highly regular functions that are very isotropic, e.g., functions from the Sobolev space $H^{d+1}(\B_1^d)$ or less regular functions that are highly anisotropic, e.g., the ridge function in \cref{eq:sobolev-ridge}.
\section{Approximation Rates in \texorpdfstring{$\RBV^2(\Omega)$}{RBV2(Omega)}} \label{sec:approximation}
A well-known result in approximation theory, first due to Maurey and Pisier~\cite{pisier1981remarques}, is that given a dictionary of atoms contained in a Hilbert space $\mathcal{H}$, the closure (with respect to the topology of $\mathcal{H}$) of the convex, symmetric hull of the dictionary is \emph{immune to the curse of dimensionality}~\cite{pisier1981remarques, jones1992simple, uat4, devore1996some, barron2008approximation}. This means that given a function $f$ in the closure of the convex, symmetric hull of the dictionary, there exists a $K$-term superposition of atoms from the dictionary $f_K$ such that $\norm{f - f_K}_\mathcal{H} \lesssim K^{-1/2}$, which does not depend on the input dimension of the function. This fact was fundamental to the approximation rates (which do not grow with the input dimension) derived for functions belonging to the spectral Barron spaces (first studied by Barron in~\cite{uat4}). 

It turns out that the unit-ball in the variation spaces of shallow neural networks can be characterized by the closure of the convex, symmetric hull of a dictionary of neural activation functions and are therefore also immune to the curse of dimensionality~\cite{convex-nn,siegel2021characterization}. We will use results from~\cite{convex-nn,siegel2021characterization} to readily derive approximation rates for functions in $\RBV^2(\Omega)$ that are immune to the curse of dimensionality. For simplicity we will suppose that $\Omega = \B_1^d$ as defined in \cref{eq:Euclidean-ball}. Similar results as those stated in the sequel can be derived for more general bounded domains $\Omega \subset \R^d$.
\begin{theorem} \label{thm:approximation}
    Given $f \in \RBV^2(\B_1^d)$, there exists a shallow ReLU network (with a skip connection) with $K$ neurons of the form in \cref{eq:ridge-spline-domain}, denoted $f_K$, such that
    \[
        \norm{f - f_K}_{L^\infty(\B_1^d)} \lesssim_d \RTV^2_{\B_1^d}(f) \, K^{-\frac{d + 3}{2d}}.
    \]
\end{theorem}
\begin{proof}
    Given $f \in \RBV^2(\B_1^d)$, we have from \cref{rem:identification} the representation
    \[
        f(\vec{x}) = \int_{\Sph^{d-1} \times [-1, 1]} \rho(\vec{w}^\T\vec{x} - b) \dd \mu(\vec{w}, b) + \vec{c}^\T\vec{x} + c_0.
    \]
    It is known that the integral in the above display can be approximated in $L^\infty(\B_1^d)$ by a superposition of $K$ ReLU neurons of the form $\vec{x} \mapsto \rho(\vec{w}^\T\vec{x} - b)$, $\vec{w} \in \Sph^{d-1}$ and $b \in [-1, 1]$, denoted $\tilde{f}_K$, with an approximation rate of
    \begin{align*}
        &\norm{\int_{\Sph^{d-1} \times [-1, 1]} \rho(\vec{w}^\T(\dummy) - b) \dd\mu(\vec{w}, b) - \tilde{f}_K}_{L^\infty(\B_1^d)} \\
        &\qquad \lesssim_d \norm{\mu}_{\M(\Sph^{d-1} \times [-1, 1])} \, K^{-\frac{d + 3}{2d}},
    \end{align*}
    We refer the reader to~\cite{MatouvsekZonotope} and~\cite[Proposition~1]{convex-nn} for this fact. Next, since $\norm{\mu}_{\M(\Sph^{d-1} \times [-1, 1])} = \RTV^2_{\B_1^d}(f)$, the result follows by choosing $f_K(\vec{x}) \coloneqq \tilde{f}_K(\vec{x}) + \vec{c}^\T\vec{x} + c_0$.
\end{proof}
\begin{remark} \label{rem:approximation-lower-bound}
    The approximation rate in \cref{thm:approximation} cannot be improved. We refer the reader to~\cite{siegel} for approximation lower bounds in the variation spaces of shallow neural networks. We also remark that since \cref{thm:approximation} holds in $L^\infty(\B_1^d)$, it also holds for any $L^p(\B_1^d)$, $1 \leq p < \infty$, where the implicit constant will depend on $d$ and $p$.
\end{remark}
\begin{remark}
     As $d \to \infty$, \cref{thm:approximation,rem:approximation-lower-bound} says that the approximation rate is $K^{-1/2}$ and is therefore immune to the curse of dimensionality.
\end{remark}

\section{Function Estimation in \texorpdfstring{$\RBV^2(\Omega)$}{RBV2(Omega)}} \label{sec:estimation}

In this section we will consider the usual setup of nonparametric regression in the \emph{fixed design} setting. Consider the problem of estimating a function $f \in \RBV^2(\Omega)$ from the noisy samples
\[
    y_n = f(\vec{x}_n) + \varepsilon_n, \: n = 1, \ldots, N,
\]
where $\curly{\varepsilon_n}_{n=1}^N$ are i.i.d. $\N(0, \sigma^2)$ random variables and $\curly{\vec{x}_n}_{n=1}^N \subset \Omega$ are fixed, but \emph{scattered}, design points. For simplicity we will suppose that $\Omega = \B_1^d$ as defined in \cref{eq:Euclidean-ball}. Similar results as those stated in the sequel can be derived for more general bounded domains $\Omega \subset \R^d$. 

\begin{theorem} \label{thm:mse-bound}
    Consider the problem of estimating a function $f \in \RBV^2(\B_1^d)$ such that $\RTV^2_{\B_1^d}(f) \leq C$ from the noisy samples
    \[
        y_n = f(\vec{x}_n) + \varepsilon_n, \: n = 1, \ldots, N,
    \]
    where $\curly{\varepsilon_n}_{n=1}^N$ are i.i.d. $\N(0, \sigma^2)$ random variables and $\curly{\vec{x}_n}_{n=1}^N \subset \B_1^d$ are fixed design points. Then, any solution to the variational problem
    \begin{equation}
        \hat{f} \in \argmin_{f \in \RBV^2(\B_1^d)} \: \sum_{n=1}^N \abs{y_n - f(\vec{x}_n)}^2 \:\subj\: \RTV^2_{\B_1^d}(f) \leq C
        \label{eq:estimator-variational}
    \end{equation}
    has a mean-squared error bound of
    \begin{equation}
        \E\sq{\frac{1}{N} \sum_{n=1}^N \abs*{f(\vec{x}_n) - \hat{f}(\vec{x}_n)}^2} \lesssim_d \tilde{O}\paren{ C^{\frac{2d}{2d + 3}} \paren{\frac{N}{\sigma^2}}^{-\frac{d + 3}{2d + 3}}},
        \label{eq:empirical-mse-bound}
    \end{equation}
    where $\tilde{O}(\dummy)$ hides universal constants and logarithmic factors, where the only random variables in the expectation above are the noise terms $\curly{\varepsilon_n}_{n=1}^N$.
\end{theorem}
\begin{remark}
    Notice that as $d \to \infty$, we have that $C^{\frac{2d}{2d + 3}} \to C$ and so the bound scales linearly with the constant $C$.
\end{remark}
The proof of \cref{thm:mse-bound} follows standard techniques (see, e.g.,~\cite[Chapter~9]{geer2000empirical} or~\cite[Chapter~13]{wainwright2019high}) based on the metric entropy of the model class
\begin{equation}
    \curly{f \in \RBV^2(\B_1^d) \st \RTV^2_{\B_1^d}(f) \leq C}
    \label{eq:model-class}
\end{equation}
with respect to the \emph{empirical} $L^2$-norm defined with respect to the sampling locations $\curly{\vec{x}_n}_{n=1}^N$
\begin{equation}
    \norm{f}_N^2 \coloneqq \frac{1}{N} \sum_{n=1}^N \abs*{f(\vec{x}_n)}^2.
    \label{eq:empirical-norm}
\end{equation}
We use our approximation rate in \cref{sec:approximation} to upper bound this metric entropy. The proof of \cref{thm:mse-bound} appears in \cref{app:mse-bound}.

\begin{remark}
    Computing an estimator that satisfies the bound in \cref{eq:empirical-mse-bound} requires finding a solution to the variational problem in \cref{eq:estimator-variational}. By \cref{thm:rep-thm-domain}, one can find a solution to the variational problem by training a sufficiently wide shallow ReLU network via gradient descent with weight decay (to a global minimizer). This is the same as finding a solution to the the non-convex neural network training problem in  \cref{eq:weight-decay}, where, by Lagrange calculus, the choice of $\lambda$ depends on $C$ and the data through the data-fitting term. An alternative approach would be to the use greedy algorithms (also known as Frank--Wolfe algorithms)~\cite{frank1956algorithm,jones1992simple,convex-nn,siegel2022optimal}.
\end{remark}

\begin{remark}
    Since when $d = 1$, $\RBV^2(\B_1^d)$ is exactly the space $\BV^2[-1, 1]$ (see the discussion in \cref{sec:bounded-domain}), the result of \cref{thm:mse-bound} recovers the well-known mean-squared error rate of $N^{-4/5}$ of locally adaptive linear spline estimators~\cite{locally-adaptive-regression-splines}.
\end{remark}

The result of \cref{thm:mse-bound} can be extended from the fixed design setting to the random design setting using standard techniques (see, e.g.,~\cite[Chapter~14]{wainwright2019high}). In particular, assuming the design points $\curly{\vec{x}_n}_{n=1}^N \subset \B_1^d$ are i.i.d. uniform random variables on $\B_1^d$, we can use the techniques outlined in~\cite[Chapter~14]{wainwright2019high} to derive the same mean-squared error rate (for sufficiently large $N$) with respect to $\norm{\dummy}_{L^2(\B_1^d; \P_X)}$, where $\P_X$ denotes the uniform probability measure on $\B_1^d$. 
This follows from the fact that the empirical norm $\norm{\dummy}_N$ concentrates to the population norm $\norm{\dummy}_{L^1(\B_1^d; \P_X)}$ at the same rate as the right-hand side of \cref{eq:empirical-mse-bound}~\cite[Chapter~14, Corollary~14.15]{wainwright2019high}. Therefore, we have the following corollary to \cref{thm:mse-bound}.
\begin{corollary} \label{cor:mse}
    Consider the problem of estimating a function $f: \B_1^d \to \R$ satisfying $\RTV^2_{\B_1^d}(f) \leq C$
    \[
        y_n = f(\vec{x}_n) + \varepsilon_n, \: n = 1, \ldots, N,
    \]
    where $\curly{\varepsilon_n}_{n=1}^N$ are i.i.d. $\N(0, \sigma^2)$ random variables and $\curly{\vec{x}_n}_{n=1}^N \subset \B_1^d$ are i.i.d. uniform random variables on $\B_1^d$. Then, for sufficiently large $N$, any solution to the variational problem
    \[
        \hat{f} \in \argmin_{f \in \RBV^2(\B_1^d)} \: \sum_{n=1}^N \abs{y_n - f(\vec{x}_n)}^2 \:\subj\: \RTV^2_{\B_1^d}(f) \leq C
    \]
    has a mean-squared error bound of
    \begin{align*}
        &\E\norm*{f - \hat{f}}_{L^2(\B_1^d; \P_X)}^2  \\
        &\qquad\qquad\lesssim_d \tilde{O}\paren{C^{\frac{2d}{2d + 3}} \paren{\frac{N}{\sigma^2}}^{-\frac{d + 3}{2d + 3}} + \paren{\frac{N}{C'}}^{-\frac{d+3}{2d + 3}}},
    \end{align*}
    where $\tilde{O}(\dummy)$ hides universal constants and logarithmic factors, $C'>0$ is a constant that depends on $C$, and $\P_X$ denotes the uniform probability measure on $\B_1^d$.
\end{corollary}
\begin{remark}
    \Cref{cor:mse} also provides an upper bound on the \emph{sampling number} for the $\RBV^2(\B_1^d)$ model class when $\sigma \to 0$. We refer the reader to~\cite{binev2022optimal} for a precise definition of sampling numbers for model classes.
\end{remark}

The following theorem shows that this mean-squared error rate cannot be improved. In other words, the rate in \cref{thm:mse-bound} is (up to logarithmic factors) minimax optimal.

\begin{theorem} \label{thm:minimax}
    Consider the problem of estimating a function $f: \B_1^d \to \R$ satisfying $\RTV^2_{\B_1^d}(f) \leq C$
    from the noisy samples
    \[
        y_n = f(\vec{x}_n) + \varepsilon_n, \: n = 1, \ldots, N,
    \]
    where $\curly{\varepsilon_n}_{n=1}^N$ are i.i.d. $\N(0, \sigma^2)$ random variables. Then, we have the following minimax lower bound
    \[
        \inf_{\hat{f}} \sup_{\substack{f \in \RBV^2(\B_1^d) \\ \RTV^2_{\B_1^d}(f) \leq C}} \E\norm*{f - \hat{f}}_{L^2(\B_1^d; \P_X)}^2 \gtrsim_d \paren{\frac{N}{\sigma^2}}^{-\frac{d + 3}{2d + 3}},
    \]
    where the $\inf$ is over all functions of the data and $\P_X$ denotes the uniform probability measure on $\B_1^d$.
\end{theorem}

The proof of \cref{thm:minimax} invokes a general result of Yang and Barron~\cite{yang1999information} regarding minimax rates over model classes. Invoking the result involves bounds on the $L^2(\B_1^d; \P_X)$-metric entropy of the model class in \cref{eq:model-class}. We can readily bound this metric entropy due to recent results which tightly bound the metric entropy of model classes in the variation space $\V^2(\B_1^d)$ from~\cite{siegel}. The proof of \cref{thm:minimax} appears in \cref{app:minimax}.

\subsection{Breaking the Curse of Dimensionality}
When $d = 1$, \cref{thm:mse-bound,thm:minimax} recovers (up to logarithmic factors) the well-known minimax rate of $N^{-4/5}$ for $\BV^2[-1, 1]$ model classes~\cite{donoho1998minimax}. On the other hand, when $d \to \infty$, the rate approaches (up to logarithmic factors) $N^{-1/2}$, and is therefore immune to the curse of dimensionality. To understand why this is happening, we recall from \cref{sec:RBV-mixed-variation} that $\RBV^2(\B_1^d)$ can be viewed as a mixed variation space.

Classical folklore in nonparametric statistics says that the minimax rate for $H^k(\B_1^d)$ model classes is $N^{-\frac{2k}{2k + d}}$. From \cref{thm:sobolev-embedding}, we have that $H^{d+1}(\B_1^d) \subset \RBV^2(\B_1^d)$. The minimax rate for $H^{d+1}(\B_1^d)$ model classes is then $N^{-\frac{2d + 2}{3d + 2}}$. As $d \to \infty$, this rate is $N^{-2/3}$. Therefore, we see that the space $H^{d+1}(\B_1^d)$ is also immune to the curse of dimensionality, but estimating functions in $H^{d+1}(\B_1^d)$ is strictly easier than estimating functions in the larger $\RBV^2(\B_1^d)$ space. This is due to the fact that $\RBV^2(\B_1^d)$ is a mixed variation space that contains highly isotropically regular functions that belong to the Sobolev space $H^{d+1}(\B_1^d)$ as well as anistropic less regular functions such as the ridge function defined in \cref{eq:sobolev-ridge}, which may only have two weak derivatives.

These observations about $\RBV^2(\B_1^d)$ make it a compelling framework for high-dimensional nonparametric estimation. Moreover, the connections with shallow ReLU networks could also shed light on the empirical success of neural networks in practice: neural networks learn functions in spaces that are immune to the curse of dimensionality.

\section{Neural Networks vs.\ Linear Methods} \label{sec:nn-not-kernel}
In this section we will illustrate the idea that the estimator studied in \cref{sec:estimation} is \emph{locally adaptive} (a term coined by Donoho and Johnstone in~\cite{donoho1998minimax}) unlike more classical \emph{linear methods} (which include kernel methods~\cite{learning-with-kernels}). We will illustrate this both quantitatively via rates for function estimation as well as qualitatively via numerical experiments. For the problem of function estimation, a linear method is a method in which the estimator is a \emph{linear} function of the data $(y_1, \ldots, y_N)$, i.e., the estimator is computed via a linear map $T: \R^N \to \F$, where $\F$ is some model class and $T$ can depend on the design points $\curly{\vec{x}_n}_{n=1}^N$ in an arbitrary way. Due to the sparsity-promoting nature of the $\M$-norm used to define $\RTV^2_{\B_1^d}(\dummy)$, the estimator in \cref{thm:mse-bound} is a \emph{nonlinear} function of the data. This is analagous to  LASSO-type estimators arising from $\ell^1$-norm regularized problems, which are nonlinear estimators for discrete-domain problems.

\subsection{The Univariate Case} \label{sec:exp-univariate}
In the univariate case, we have from \cref{rem:1D-spaces} that the variational problem in \cref{eq:estimator-variational} reduces to the (regularized) variational problem
\begin{equation}
    \min_{f \in \BV^2[-1, 1]} \sum_{n=1}^N \abs*{y_n - f(x_n)}^2 + \lambda \, \norm*{\D^2 f}_{\M[-1, 1]},
    \label{eq:locally-adaptive-spline-variational}
\end{equation}
where $\lambda > 0$ is the regularization parameter.  The solutions are locally adaptive linear spline estimators~\cite{locally-adaptive-regression-splines}. It is known that the minimax rate for $\BV^2[-1, 1]$ model classes is $N^{-4/5}$~\cite{donoho1998minimax}, which is achieved by the locally adaptive linear spline estimator~\cite{locally-adaptive-regression-splines}. Moreover, when restricted to \emph{linear estimators}, the linear minimax rate is known to be $N^{-3/4}$~\cite{donoho1998minimax}, which is achieved (up to logarithmic factors) by the cubic smoothing spline estimator~\cite{splines-minimum,spline-rep-thm}.  The cubic smoothing spline is a solution to the variational problem
\begin{equation}
    \min_{f \in H^2[-1, 1]} \sum_{n=1}^N \abs*{y_n - f(x_n)}^2 + \lambda \, \norm*{\D^2 f}_{L^2[-1, 1]}^2,
    \label{eq:smoothing-spline-variational}
\end{equation}
where
\[
    H^2[-1, 1] \coloneqq \curly{f \in \mathscr{D}'[-1,1] \st \norm*{\D^2 f}_{L^2[-1, 1]} < \infty},
\]
is the second-order $L^2$-Sobolev space and $\mathscr{D}'[-1, 1]$ denotes the space of distributions (generalized functions) on $[-1, 1]$. Moreover, we have the strict containment $H^2[-1, 1] \subset \BV^2[-1, 1]$. The key difference between the problem in \cref{eq:locally-adaptive-spline-variational} and the problem in \cref{eq:smoothing-spline-variational} is the difference between the \emph{sparsity-promoting} $\M$-norm regularization in \cref{eq:locally-adaptive-spline-variational} and the $L^2$-norm regularization in \cref{eq:smoothing-spline-variational}. This is analogous to the difference between $\ell^1$-norm and $\ell^2$-norm regularization in discrete-domain problems.

The main takeaway message here is that this difference \emph{quantifies} a fundamental gap between neural network estimators and any linear/kernel estimator; the gap between the rates $N^{-4/5}$ and $N^{-3/4}$. The reason for this gap is that functions in $\BV^2[-1, 1]$ are \emph{spatially inhomogeneous}, while functions in $H^2[-1, 1]$ are \emph{spatially homogeneous}. Neural network estimators are able to adapt to the inhomogeneities of the data-generating function (and are therefore \emph{locally adaptive}), while linear methods cannot. This shows that even the simplest neural networks (shallow, univariate)  \emph{outperform} linear methods when the data-generating function is spatially inhomogeneous. We illustrate this phenomenon in \cref{fig:1D}, where we consider the problem of fitting data generated from a spatially inhomogenous function in $\BV^2[-1, 1]$ that is not in $H^2[-1, 1]$ using a shallow ReLU network and a cubic smoothing spline. As these results are qualitative, we manually adjusted the regularization parameter $\lambda$ in the experiments in order to find solutions that visually capture the phenomenon described above. The code to generate \cref{fig:1D} is publicly available\footnote{\url{https://github.com/rp/estimation-shallow-relu}}.

\begin{figure}[htb]
    \centering
    \centerline{\includegraphics[width=\columnwidth]{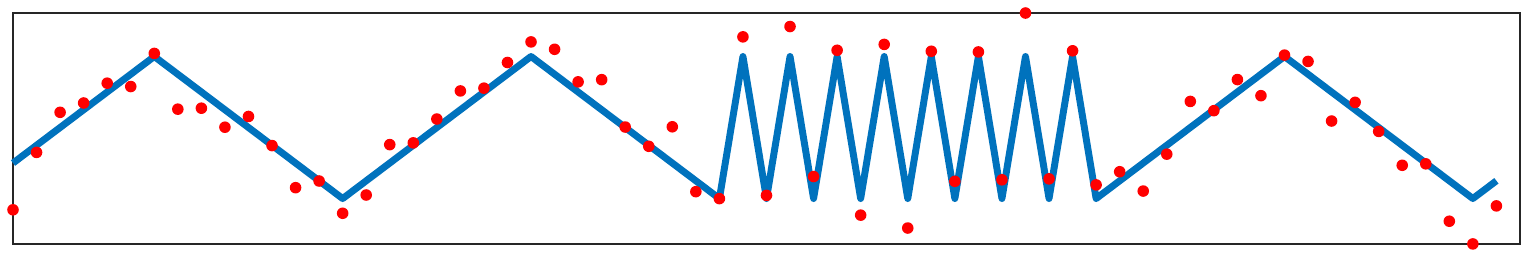}}
      (a) True function and data. \smallskip
      
     \centerline{\includegraphics[width=\columnwidth]{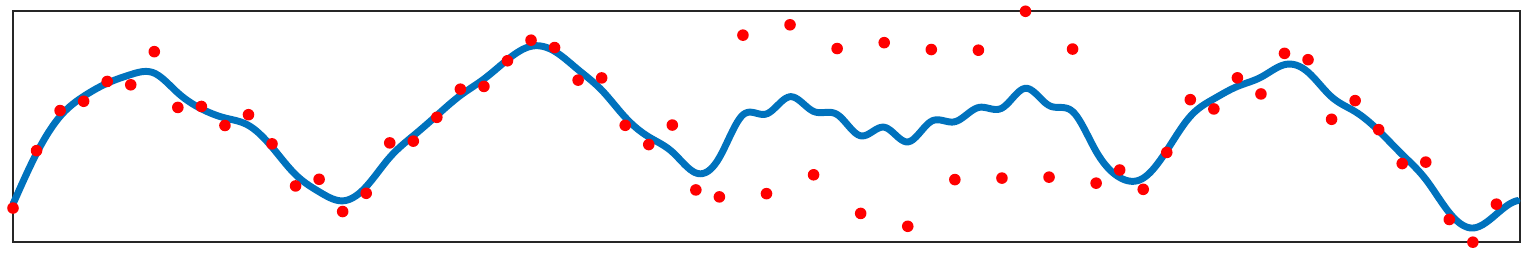}}
     (b) Cubic smoothing spline with large $\lambda$. \smallskip
      
     \centerline{\includegraphics[width=\columnwidth]{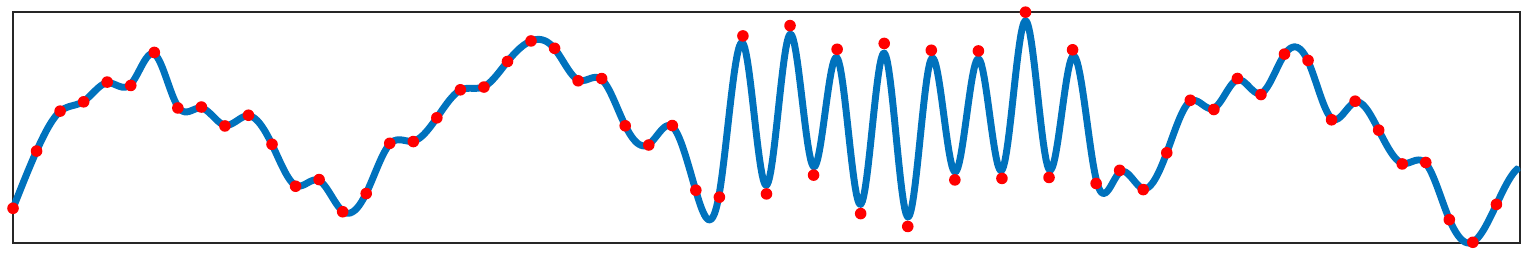}}
     (c) Cubic smoothing spline with small $\lambda$. \smallskip
      
     \centerline{\includegraphics[width=\columnwidth]{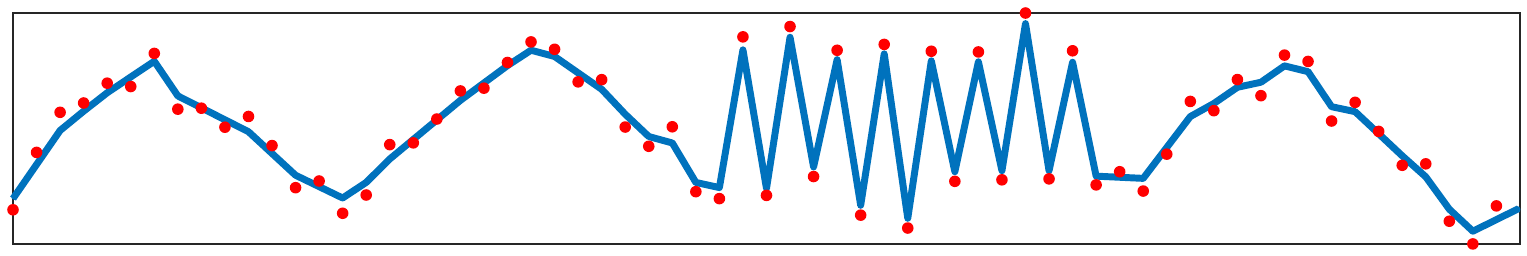}}
     (d) Shallow ReLU network or locally adaptive linear spline. \smallskip
     
    \caption{In (a) we generate data from noisy samples of a function in $\BV^2[-1, 1]$ but not in $H^2[-1, 1]$. In (b) and (c) we fit the data using a cubic smoothing spline with both large and small $\lambda$. In (d) we fit the data using a locally adaptive linear spline which corresponds to training a shallow ReLU network (to a global minimizer) with weight decay (or path-norm regularization).}
    \label{fig:1D}
\end{figure}

In \cref{fig:1D}(a) we plot a function (in blue) and generate a data set by taking noisy samples (in red) of the function plus i.i.d. Gaussian noise. Clearly this function is in $\BV^2[-1, 1]$ but not in $H^2[-1, 1]$ since taking two (distributional) derivatives of this function is an impulse train. This function is spatially inhomogeneous since it is highly oscillatory in some regions and less oscillatory in others.

In \cref{fig:1D}(b) and \cref{fig:1D}(c), we plot the cubic smoothing spline fit to the data for large and small $\lambda$, respectively. This illustrates that the cubic smoothing spline (which is a kernel method) \emph{cannot} adapt to the spatial inhomogenity of the underlying function. Even by adjusting the regularization parameter $\lambda$, the solution cannot adapt to the spatial inhomogeneity of the underlying function. Indeed, we see for large $\lambda$ in \cref{fig:1D}(b) that the cubic smoothing spline oversmooths the high variation portion of the data and we see for small $\lambda$ in \cref{fig:1D}(c) that the cubic smoothing spline undersmooths (overfits) the low variation portion of the data.

In \cref{fig:1D}(d) we plot a solution to the variational problem in \cref{eq:locally-adaptive-spline-variational}, which is a locally adaptive linear spline which can be computed by training a shallow ReLU network (to a global minimizer) with weight decay or path-norm regularization. In this case, we see that the locally adaptive linear spline is able to adapt to the spatial inhomogeneities of the underlying function.

We also remark that wavelet shrinkage estimators, in which the mother wavelet is sufficiently regular, are also a minimax optimal estimators for nonparametric estimation of $\BV^2[-1, 1]$ functions~\cite{donoho1998minimax}. This shows that in the simplest setting, shallow ReLU networks trained with weight decay (to a global minimizer) perform exactly the same as classical techniques such as locally adaptive spline estimators and wavelet shrinkage estimators.

\subsection{The Multivariate Case} \label{sec:exp-multivariate}

In the multivariate case, we see a similar gap from the univariate case. In particular, we derive the following linear minimax lower bound for the estimation problem over $\RBV^2(\B_1^d)$.

\begin{theorem} \label{thm:linear-minimax}
    Consider the problem of estimating a function $f \in \RBV^2(\B_1^d)$ satisfying $\RTV^2_{\B_1^d}(f) \leq C$ from the noisy samples
    \[
        y_n = f(\vec{x}_n) + \varepsilon_n, \: n = 1, \ldots, N,
    \]
    where $\curly{\varepsilon_n}_{n=1}^N$ are i.i.d. $\N(0, \sigma^2)$ random variables and $\curly{\vec{x}_n}_{n=1}^N \subset \B_1^d$ are i.i.d. uniform random variables on $\B_1^d$. Then, for sufficiently large $N$, we have the following linear minimax lower bound
    \[
        \inf_{\hat{f} \text{ linear}} \sup_{\substack{f \in \RBV^2(\B_1^d) \\ \RTV^2_{\B_1^d}(f) \leq C}} \E\norm*{f - \hat{f}}_{L^2(\B_1^d; \P_X)}^2 \gtrsim_d \paren{\frac{N}{\sigma^2}}^{-\frac{3}{d + 3}},
    \]
    where the $\inf$ is over all \emph{linear} functions of the data and $\P_X$ denotes the uniform probability measure on $\B_1^d$.
\end{theorem}

The proof of \cref{thm:linear-minimax} appears in \cref{app:linear-minimax} and hinges on several results from ridgelet analysis developed by Cand{\`e}s~\cite{candes-phd,candes2003ridgelets}. Just as in the univariate case, the takeaway message here is that this lower bound quantifies a fundamental gap between neural network estimators and any linear/kernel estimator. The minimax rates for nonlinear and linear estimation are  $N^{-\frac{d+3}{2d+3}}$ and $N^{-\frac{3}{d+3}}$, respectively. As $d \to \infty$, the nonlinear estimation rate tends to $N^{-1/2}$, which is immune to the curse of dimensionality, while the linear estimation rate suffers the curse of dimensionality. Moreover, these rates recover the univariate ($d=1$) rates of $N^{-4/5}$ and $N^{-3/4}$. The reason for the gap between the nonlinear and linear minimax rates is that functions in $\RBV^2(\B_1^d)$ are \emph{spatially inhomogeneous} since it is a mixed variation space and neural network estimators are able to adapt to the inhomogeneities of the data-generating function (and are therefore \emph{locally adaptive}), while linear methods cannot.

\begin{figure*}[htb]
    \begin{minipage}[b]{0.3\linewidth}
        \centering
        \centerline{\includegraphics[width=\textwidth]{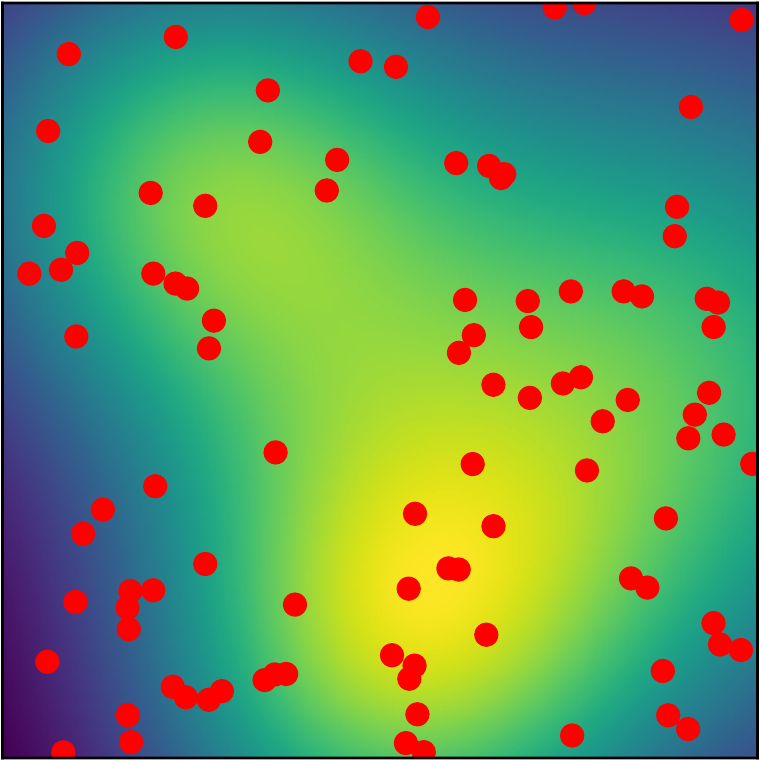}}
        (a) True function and data.
    \end{minipage}
    \hfill
    \begin{minipage}[b]{0.3\linewidth}
        \centering
        \centerline{\includegraphics[width=\textwidth]{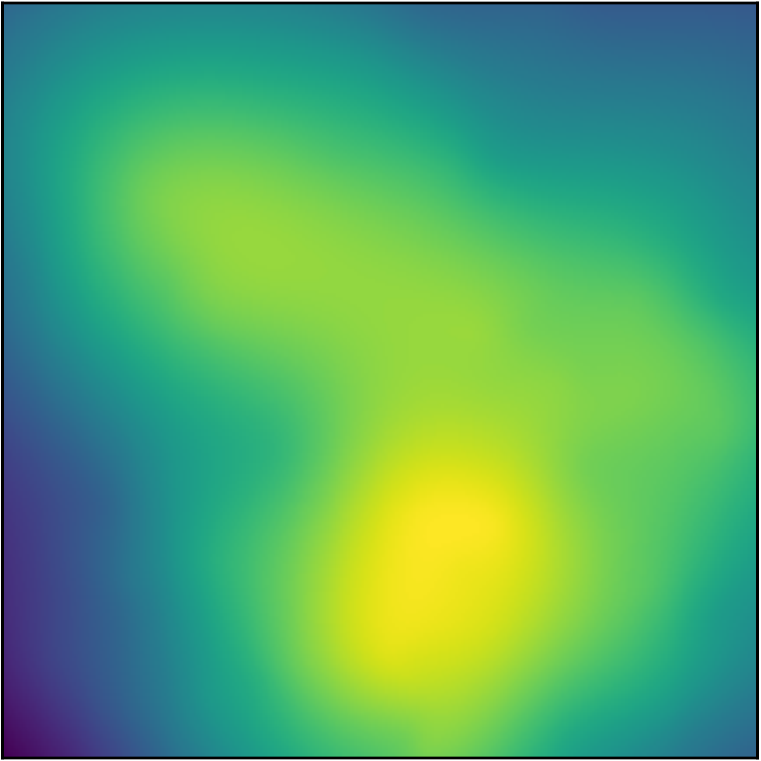}}
        (b) Thin-plate spline.
    \end{minipage}
    \hfill
    \begin{minipage}[b]{0.3\linewidth}
        \centering
        \centerline{\includegraphics[width=\textwidth]{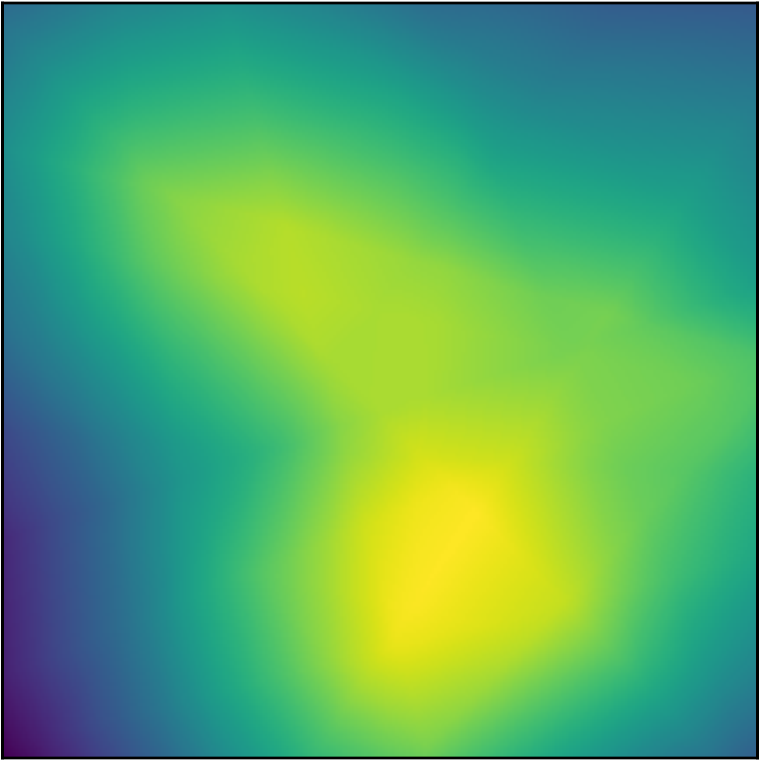}}
        (c) Shallow ReLU network.
    \end{minipage}
    
    \caption{In (a) we generate noisy samples of a function in both $\RBV^2(\B_1^2)$ and $H^2(\B_1^2)$. In (b) we fit the data using a thin-plate spline. In (c) we fit the data with a shallow ReLU network trained with weight decay.}
    \label{fig:2D-smooth}
\end{figure*}

\begin{figure*}[htb]
    \begin{minipage}[b]{0.3\linewidth}
        \centering
        \centerline{\includegraphics[width=\textwidth]{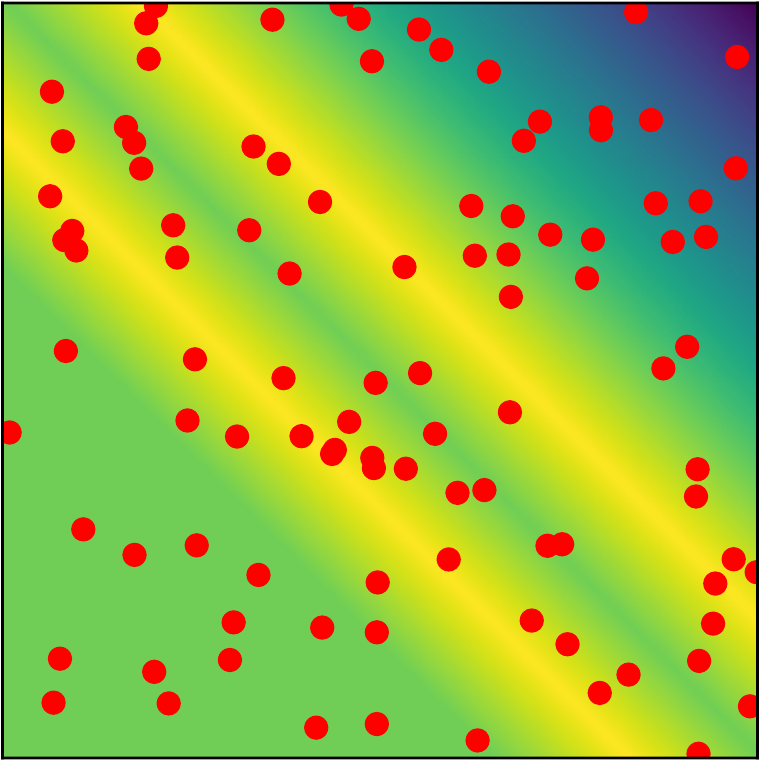}}
        (a) True function and data.
    \end{minipage}
    \hfill
    \begin{minipage}[b]{0.3\linewidth}
        \centering
        \centerline{\includegraphics[width=\textwidth]{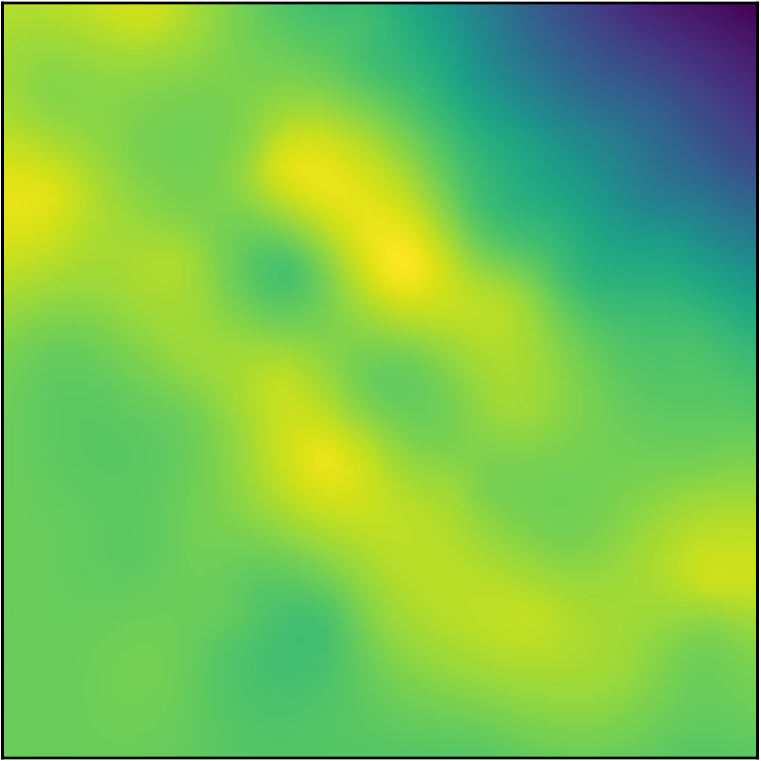}}
        (b) Thin-plate spline.
    \end{minipage}
    \hfill
    \begin{minipage}[b]{0.3\linewidth}
        \centering
        \centerline{\includegraphics[width=\textwidth]{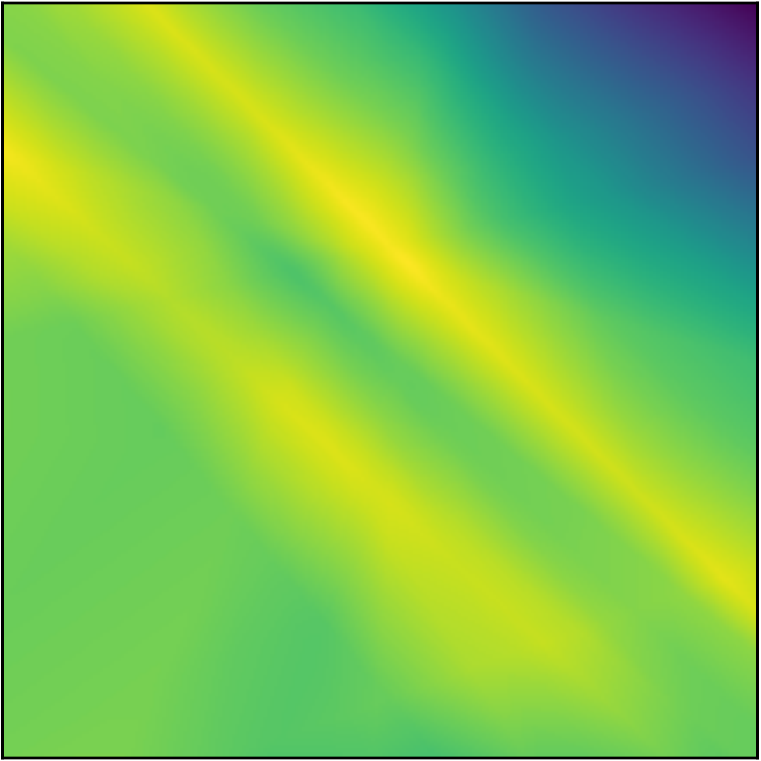}}
        (c) Shallow ReLU network.
    \end{minipage}
    
    \caption{In (a) we generate noisy samples of a function in $\RBV^2(\B_1^2)$ but not in $H^2(\B_1^2)$. In (b) we fit the data using a thin-plate spline. In (c) we fit the data with a shallow ReLU network trained with weight decay.}
    \label{fig:2D-ridge}
\end{figure*}

We illustrate this phenomenon by considering the problem of estimating a two-dimensional function and compare solutions to the variational problem in \cref{eq:estimator-variational} with the thin-plate spline estimator~\cite{spline-models-observational}, which is a linear method and a special case of a kernel method. The thin-plate spline is a solution to the variational problem
\begin{align*}
    &\min_{f \in H^2(\B_1^2)} \sum_{n=1}^N \abs*{y_n - f(\vec{x}_n)}^2 \\
    &\quad + \lambda \paren{\norm{\partial_{x_1}^2 f}_{L^2(\B_1^2)}^2 + 2 \norm{\partial_{x_2}\partial_{x_1} f}_{L^2(\B_1^2)}^2 + \norm{\partial_{x_2}^2 f}_{L^2(\B_1^d)}^2},
\end{align*}
where $H^2(\B_1^2)$ is the second-order $L^2$-Sobolev space, which is defined as the space of all functions where the regularizer in the above display is finite. Notice that the problem in the above  is  a generalization of the cubic smoothing spline problem in \cref{eq:smoothing-spline-variational}. We compare the shallow ReLU network estimator to the thin-plate spline estimator for two functions, one that is in both $\RBV^2(\B_1^2)$ and $H^2(\B_1^2)$, and one that is only in $\RBV^2(\B_1^2)$. In all the experiments, we manually adjusted the regularization parameter $\lambda$ to obtain the best results for each method.  Thus, the results (visually) compare the best performance of each method.

In \cref{fig:2D-smooth} we consider a function that is a superposition of three Gaussians. This function is infinitely differentiable 
and therefore in both $\RBV^2(\B_1^d)$ and $H^2(\B_1^2)$. In \cref{fig:2D-smooth}(a), we plot the function with a heatmap where lighter colors correspond to larger values and darker colors correspond to smaller values. We then generate a data set by taking noisy samples (in red) of the function plus i.i.d. Gaussian noise. In \cref{fig:2D-smooth}(b), we plot the heatmap of the thin-plate spline fit to the data. We see that the thin-plate spline estimates the original function quite well. In \cref{fig:2D-smooth}(c), we plot the heatmap of the shallow ReLU network. We also see that the shallow ReLU network estimates the original function quite well.

In \cref{fig:2D-ridge} we consider a function that is a ridge function in a random direction where the profile is a continuous piecewise-linear function, a triangular waveform. This function does not have two weak derivatives and is therefore not in $H^2(\B_1^2)$, but is in $\RBV^2(\B_1^d)$. In \cref{fig:2D-ridge}(a), we plot the heatmap of the function. We then generate a data set by taking noisy samples (in red) of the function plus i.i.d. Gaussian noise. In \cref{fig:2D-ridge}(b), we plot the heatmap of the thin-plate spline fit to the data. We see that the thin-plate spline struggles to estimate the original function. In \cref{fig:2D-ridge}(c), we plot the heatmap of the shallow ReLU network. We see that the shallow ReLU network estimates the original function quite well. 

The main takeaway message here is that the shallow ReLU network is able to \emph{locally adapt} to the mixed variation of the data-generating function, whether it be a highly isotropically regular function or a anistropically less regular function, while linear/kernel methods cannot. 
The code to generate \cref{fig:2D-smooth,fig:2D-ridge} is publicly available\footnote{\url{https://github.com/rp/estimation-shallow-relu}}.

\begin{remark}
    We believe that the results of \cref{sec:exp-univariate,sec:exp-multivariate} provide compelling evidence that trying to understand neural networks via linearization schemes such as the neural tangent kernel~\cite{ntk} do not properly capture what neural networks are actually doing in practice. The key idea being that neural networks are able to locally adapt to the mixed variation of the underlying data-generating function.
\end{remark}
\section{Conclusion}
In this paper we studied the problem of estimating an unknown function defined on a bounded domain $\Omega \subset \R^d$ from $\RBV^2(\Omega)$, the natural function space of shallow ReLU networks, from noisy samples. We studied the estimators that correspond to training a shallow ReLU network with weight decay (or path-norm regularization) to a global minimizer. We showed that these estimators provide (up to logarithmic factors) minimax optimal rates of convergence for $\RBV^2(\Omega)$ model classes. Moreover, these rates were immune to the curse of dimensionality. We showed that $\RBV^2(\Omega)$ contains highly isotropically regular functions that belong to the Sobolev space $H^{d+1}(\Omega)$ as well as anisotropic less regular functions, and therefore can be viewed as mixed variation spaces, giving insight into why shallow ReLU network estimators are immune to the curse of dimensionality. In particular, we quantify an explicit gap between linear and nonlinear methods and show that linear methods are suboptimal for estimating functions in $\RBV^2(\Omega)$.

There are a number of open questions that may be asked. For example, considering higher-order variants of $\RBV^2(\Omega)$. Our previous work in~\cite{ridge-splines} also studied the higher-order variants defined on $\R^d$, $\RBV^m(\R^d)$, where $m \geq 2$ is an integer. These higher-order spaces are defined by the seminorm $\RTV^m(\dummy)$, which corresponds to replacing $\partial_t^2$ with $\partial_t^m$ in \cref{eq:RTV} and considering a different growth restriction than in \cref{eq:RBV}. These higher-order spaces correspond to shallow neural networks with activation functions that are the $(m - 1)$th power of the ReLU. Although many of the results in this paper are straightforward to generalize to $\RBV^m$-spaces, some of the results are also very specific to $\RBV^2$-spaces. In particular, it is currently an open question on whether or not similar approximation rates as in \cref{thm:approximation} can be derived in $L^\infty(\B_1^d)$. Using results from~\cite{siegel}, we can derive similar optimal approximation rates in $L^2(\B_1^d)$, but the mean-squared error rates hinged on the $L^\infty(\B_1^d)$ approximation rates. Finally, perhaps the most important open question regards estimation with deep ReLU networks fit to data. Our prior work in~\cite{deep-ridge-splines} developed a deep variant of $\RBV^2(\R^d)$, and derived a representer theorem for deep ReLU networks. This deep $\RBV^2$-space could provide the right framework for nonparametric estimation with deep ReLU networks.

\section*{Acknowledgment}
The authors would like to thank Jonathan Siegel for answering many questions about his recent work in~\cite{siegel2021characterization,siegel}. The authors would also like to thank Ronald DeVore for helpful discussions.

\appendices

\section{Proof of \cref{lemma:RBV-domain-identification}} \label[app]{app:extension-restriction}
The proof of \cref{lemma:RBV-domain-identification} relies on the direct-sum decomposition of the space $\RBV^2(\R^d)$ from our previous work in~\cite{ridge-splines}.

\subsection{The Direct-Sum Decomposition of \texorpdfstring{$\RBV^2(\R^d)$}{RBV2(Rd)}}
It was shown in~\cite[Theorem~22]{ridge-splines} that $\RBV^2(\R^d)$ is a non-reflexive Banach space, in particular, it is a Banach space with a sparsity-promoting norm. In this section we will summarize the relevant results from~\cite{ridge-splines} about the Banach structure of $\RBV^2(\R^d)$. We first remark that the space $\RBV^2(\R^d)$ as defined in \cref{eq:RBV} is defined by a \emph{seminorm} $\RTV^2(\dummy)$. The null space of this seminorm on $\RBV^2(\R^d)$ is the space of affine functions, i.e., polynomials of degree strictly less than $2$, on $\R^d$, denoted by $\mathcal{P}_{1}(\R^d)$. In~\cite{ridge-splines}, we equip $\RBV^2(\R^d)$ with a \emph{bona fide} norm by considering an arbitrary \emph{biorthogonal system} for $\mathcal{P}_1(\R^d)$.

\begin{definition} \label[defn]{defn:biorthogonal-system}
    Let $\N$ be a finite-dimensional space with $N_0 \coloneqq \dim \N$. The pair $(\vec{\phi}, \vec{p}) = \curly{(\phi_n, p_n)}_{n=1}^{N_0}$ is called a \emph{biorthogonal system} for $\mathcal{N}$ if $\vec{p} = \curly{p_n}_{n=1}^{N_0}$ is a basis of $\N$ and the ``boundary'' functionals $\vec{\phi} = \curly{\phi_n}_{n=1}^{N_0}$ with $\phi_n \in \N'$ (the continuous dual of $\N$) satisfy the biorthogonality condition $\ang{\phi_k, p_n} = \delta[k - n]$, $k, n = 1, \ldots, N_0$, where $\delta[\dummy]$ is the Kronecker impulse.
\end{definition}

Recall from \cref{eq:RTV} that
\[
    \RTV^2(f) = c_d \norm*{\partial_t^2 \ramp^{d-1} \RadonOp f}_{\M(\cyl)}.
\]
For brevity, put
\[
    \ROp \coloneqq c_d \,\partial_t^2 \ramp^{d-1} \RadonOp,
\]
i.e., $\RTV^2(f) = \norm{\ROp f}_{\M(\cyl)}$. Also, note that $\dim \mathcal{P}_{1}(\R^d) = d + 1$.

\begin{proposition}[{see~\cite[Lemma~21~and~Theorem~22]{ridge-splines}}] \label[prop]{prop:direct-sum-inverse}
    Let $(\vec{\phi}, \vec{p})$ be a biorthogonal system for $\mathcal{P}_1(\R^d)$. Then, every $f \in \RBV^2(\R^d)$ has the unique direct-sum decomposition
    \begin{equation}
        f = \ROp^{-1}_\vec{\phi}\mu + q,
        \label{eq:direct-sum}
    \end{equation}
    where $\mu = \ROp_m f \in \M(\cyl)$ is an even measure\footnote{i.e., $\dd \mu(\vec{z}) = \dd \mu(-\vec{z})$.}, $q = \sum_{k=1}^{d+1} \ang{\phi_k, f}p_k \in \mathcal{P}_1(\R^d)$, and 
    \begin{equation}
        \ROp^{-1}_\vec{\phi}: \mu \mapsto \int_{\cyl} g_\vec{\phi}(\dummy, \vec{z}) \dd\mu(\vec{z}),
        \label{eq:right-inverse}
    \end{equation}
    where 
    \begin{equation}
        g_\vec{\phi}(\vec{x}, \vec{z}) =  r_\vec{z}(\vec{x}) - \sum_{k=1}^{d+1} p_k(\vec{x}) q_k(\vec{z}),
        \label{eq:kernel-of-inverse}
    \end{equation}
    where $r_\vec{z} = r_{(\vec{w}, b)} = \rho(\vec{w}^\T(\dummy) - b)$, where $\rho$ is the ReLU, and $q_k(\vec{z}) \coloneqq \ang{\phi_k, r_\vec{z}}$, where $\vec{z} = (\vec{w}, b) \in \cyl$.
\end{proposition}

The operator $\ROp^{-1}_\vec{\phi}$ defined in \cref{eq:right-inverse} has several useful properties (see~\cite[Theorem~22, Items~1~and~2]{ridge-splines}). In particular, it is a stable (i.e., bounded) right-inverse of $\ROp$ and, when restricted to
\[
    \RBV^2_\vec{\phi}(\R^d) \coloneqq \curly{f \in \RBV^2(\R^d) \st \vec{\phi}(f) = \vec{0}},
\]
it is the \emph{bona fide} inverse of $\ROp$ when restricted to the subspace of even measures in $\M(\cyl)$. The space $\RBV^2_\vec{\phi}(\R^d)$ is a concrete transcription of the abstract quotient $\RBV^2(\R^d) / \mathcal{P}_1(\R^d)$. Additionally we have from \cref{prop:direct-sum-inverse} the direct-sum decomposition $\RBV^2(\R^d) \cong \RBV^2_\vec{\phi}(\R^d) \oplus \mathcal{P}_1(\R^d)$, where $\RBV^2_\vec{\phi}(\R^d)$ is a Banach space when equipped with the norm $f \mapsto \norm{\ROp f}_{\M(\cyl)}$ and $\mathcal{P}_1(\R^d)$ is a Banach space when equipped with the norm $f \mapsto \norm{\vec{\phi}(f)}_1$. We also remark that the construction of $\ROp_\vec{\phi}^{-1}$ guarantees orthogonality of the two components in \cref{eq:direct-sum} and the biorthogonal system $(\vec{\phi}, \vec{p})$ guarantees unicity. This leads the following result equipping $\RBV^2(\R^d)$ with a norm to provide a Banach space structure.

\begin{proposition}[{see~\cite[Theorem~22, Item~3]{ridge-splines}}] \label{prop:RBV-norm}
    Let $(\vec{\phi}, \vec{p})$ be a biorthogonal system for $\mathcal{P}_1(\R^d)$. Then, $\RBV^2(\R^d)$ equipped with the norm
    \[
        \norm{f}_{\RBV^2(\R^d)} \coloneqq \RTV^2(f) + \norm{\vec{\phi}(f)}_1,
    \]
    where $\vec{\phi}(f) = (\ang*{\phi_1, f}, \ldots, \ang*{\phi_{d+1}, f}) \in \R^{d+1}$, is a Banach space.
\end{proposition}

With these results we can now prove \cref{lemma:RBV-domain-identification}.
\begin{proof}[Proof of \cref{lemma:RBV-domain-identification}]
    Given $f \in \RBV^2(\Omega)$ suppose there exists an extension $\tilde{f}_\mathsf{ext}$ such that $\eval{\tilde{f}_\mathsf{ext}}_\Omega = f$ and $\RTV^2_\Omega(f) = \RTV^2(\tilde{f}_\mathsf{ext})$ with direct-sum decomposition
    \begin{equation}
        \tilde{f}_\mathsf{ext} = \int_\cyl g_\vec{\phi}(\dummy, \vec{z}) \dd\tilde{\mu}(\vec{z}) + \tilde{q},
        \label{eq:direct-sum-not-inf}
    \end{equation}
    such that $\supp \tilde{\mu} \not\subset Z_\Omega$. Next, notice that given $g_\vec{\phi}(\dummy, \vec{z})$, where $\vec{z} \not\in Z_\Omega$, we have that $\eval{g_\vec{\phi}(\dummy, \vec{z})}_\Omega$ is an affine function. Therefore, we can find another extension $f_\mathsf{ext}$ such that $\eval{f_\mathsf{ext}}_\Omega = f$ where $\RTV^2(f_\mathsf{ext}) < \RTV^2(\tilde{f}_\mathsf{ext}) = \norm{\tilde{\mu}}_{\M(\cyl)}$ by absorbing every $g_\vec{\phi}(\dummy, \vec{z})$ where $\vec{z} \not\in Z_\Omega$ in the integrand of \cref{eq:direct-sum-not-inf} into the affine term in the direct-sum decomposition so that the restriction to $\Omega$ stays the same, a contradiction. Therefore, there exists an extension $f_\mathsf{ext} \in \RBV^2(\R^d)$ that admits an integral representation
    \begin{equation}
        f_\mathsf{ext}(\vec{x}) = \int_\cyl g_\vec{\phi}(\vec{x}, (\vec{w}, b)) \dd\mu(\vec{w}, b) + q(\vec{x})
        \label{eq:identification-g_phi}
    \end{equation}
    such that $\supp \mu \subset Z_\Omega$, where $\mu$ is an \emph{even} measure and $q$ is an affine function.
    
    Next, since $\Omega \subset \R^d$ is a bounded domain, $Z_\Omega \subset \cyl$ is also a bounded domain. Therefore, since $\supp \mu \subset Z_\Omega$, we can write
    \begin{equation}
        f_\mathsf{ext}(\vec{x}) = \int_{Z_\Omega} \rho(\vec{w}^\T\vec{x} - b) \dd\mu(\vec{w}, b) + \tilde{q}(\vec{x}),
        \label{eq:integral-representation-RBV}
    \end{equation}
    where we combine the affine terms from $g_\vec{\phi}$ (defined in \cref{eq:kernel-of-inverse}) and $q$ into the new affine function $\tilde{q}$. Moreover, with the above representation we have that $\RTV^2_{\Omega}(f) = \norm{\mu}_{\M(Z_\Omega)}$. We also remark that although $\mu$ is even from \cref{prop:direct-sum-inverse}, we can replace $\mu$ with a generic, i.e., not restricted to being even, measure $\tilde{\mu} \in \M(Z_\Omega)$ by noting that integrating against an even measure in \cref{eq:identification-g_phi} corresponds to integrating against a generic measure by considering the activation function $\rho = \abs{\dummy}$. Then, since $\abs{\dummy}$ and $\max\curly{0, \dummy}$ only differ by an affine function, we can absorb this difference for every neuron in the integrand with and the affine function $\tilde{q}$ into a new affine function $\tilde{\tilde{q}}$. Finally, this generic, i.e., not even, measure has the same $\M$-norm as the even measure.
\end{proof}

\section{Proof of \cref{thm:rep-thm-domain}} \label[app]{app:rep-thm-domain}
The proof of \cref{thm:rep-thm-domain} relies on notation introduced in \cref{app:extension-restriction}.
\begin{proof}
    Let $(\vec{\phi}, \vec{p})$ be a biorthogonal system for $\mathcal{P}_1(\R^d)$. From the proof of \cref{lemma:RBV-domain-identification}, we can identify functions in $\RBV^2(\B_1^d)$ with integral representations as in \cref{eq:identification-g_phi}. Therefore, we can instead consider the variational problem
    \[
        \min_{\substack{f \in \RBV^2(\R^d) \\ f = \ROp^{-1}_\vec{\phi}\mu + q \\ \supp \mu \subset \Sph^{d-1} \times [-1, 1]}} \: \sum_{n=1}^N \ell(y_n, f(\vec{x}_n)) + \lambda \, \norm{\ROp f}_{\M(\Sph^{d-1} \times [-1, 1])}.
    \]
    The restrictions of the functions in the solution set of the above display to $\B_1^d$ will then correspond to the solution set of the problem in \cref{eq:variational-problem-domain}. Next, we remark that the proof is identical to the proof of \cref{prop:rep-thm} (which is a special case of our prior work in~\cite[Theorem~1]{ridge-splines}). This is because the proof of~\cite[Theorem~1]{ridge-splines} boiled down to the fact that $\cyl$ is locally compact. Since $\Sph^{d-1} \times [-1, 1]$ is also locally compact, the same proof holds.
\end{proof}

\section{Proof of \cref{thm:sobolev-embedding}} \label[app]{app:sobolev-embedding}

\begin{proof}   
    Since $\B_1^d$ has a Lipschitz boundary, there exists a bounded extension operator
    \[
        \Ext: W^{d+1, 1}(\B_1^d) \to W^{d+1, 1}(\R^d),
    \]
    where we refer the reader to~\cite{calderon1961lebesgue} or~\cite[Chapter~VI]{stein1970singular} for explicit constructions of this operator. Therefore, for $f \in W^{d+1, 1}(\B_1^d)$,
    \[
        \norm{\Ext f}_{W^{d+1, 1}(\R^d)} \lesssim_d \norm{f}_{W^{d+1, 1}(\B_1^d)}.
    \]

    Given $f \in W^{d+1, 1}(\R^d)$, it was shown in~\cite{function-space-relu} that
    \[
        \RTV^2(f) \lesssim_d \norm{f}_{W^{d+1, 1}(\R^d)}.
    \]
    
    Next, we have from the definition of $\RTV^2_{\B_1^d}(\dummy)$ in \cref{eq:RTV-domain} that given any $g \in \RBV^2(\R^d)$,
    \[
        \RTV^2_{\B_1^d}\paren{\eval{g}_{\B_1^d}} \leq \RTV^2(g).
    \]
    Therefore, for any $f \in W^{d+1, 1}(\B_1^d)$,
    \begin{align*}
        \RTV^2_{\B_1^d}(f) &\leq \RTV^2(\Ext f) \\
        &\lesssim_d \norm{\Ext f}_{W^{d+1, 1}(\R^d)} \\
        &\lesssim_d \norm{f}_{W^{d+1, 1}(\B_1^d)}.
    \end{align*}
    The result then follows from the fact that $L^2(\B_1^d)$ is continuously embedded in $L^1(\B_1^d)$.
\end{proof}

\section{Proof of \cref{thm:mse-bound}} \label[app]{app:mse-bound}
To prove \cref{thm:mse-bound}, we will use the general result regarding nonparametric least squares estimators from~\cite[Chapter~13]{wainwright2019high}. This general result follows from Theorem~13.5 and the remarks following, the discussion on pg. 424, and Corollary~13.7 in~\cite[Chapter~13]{wainwright2019high}. We summarize this general result in the following proposition. 

\begin{proposition}[{see~\cite[Chapter~13]{wainwright2019high}}] \label{prop:general-mse-bound}
    Let $\F$ be a \emph{convex} model class that contains the constant function, i.e., $f \equiv 1 \in \F$. Given $f \in \F$, consider the problem of estimating $f$ from the noisy samples
    \[
        y_n = f(\vec{x}_n) + \varepsilon_n, \: n = 1, \ldots, N,
    \]
    where $\curly{\varepsilon_n}_{n=1}^N$ are i.i.d. $\N(0, \sigma^2)$ random variables and $\curly{\vec{x}_n}_{n=1}^N$ are fixed design points in the domain of $f$. Then, assuming a solution exists, any solution to the nonparametric least-squares problem
    \[
        \hat{f} \in \argmin_{f \in \F} \sum_{n=1}^N \abs*{y_n - f(\vec{x}_n)}^2
    \]
    has a mean-squared error bound of
    \[
        \E\norm*{f - \hat{f}\,}_N^2
        \lesssim \delta_N^2,
    \]
    where $\norm{\dummy}_N$ is defined in \cref{eq:empirical-norm} and $\delta_N = \delta$ satisfies the inequality
    \begin{equation}
        \frac{16}{\sqrt{N}} \int_{\frac{\delta^2}{2\sigma^2}}^\delta \sqrt{\log \N(t, \partial\F, \norm{\dummy}_N)} \dd t \leq \frac{\delta^2}{4\sigma},
        \label{eq:entropy-inequality}
    \end{equation}
    where $\N(t, \partial\F, \norm{\dummy}_N)$ denotes the $t$-covering number of the metric space $(\partial\F, \norm{\dummy}_N)$ and
    \[
        \partial\F = \F - \F = \curly{f_1 - f_2 \st f_1, f_2 \in \F}.
    \]
\end{proposition}

We will now use \cref{prop:general-mse-bound} to prove \cref{thm:mse-bound}.

\begin{proof}[Proof~of~\cref{thm:mse-bound}]
In \cref{thm:mse-bound}, our model class is
\begin{equation}
    \F_C \coloneqq \curly{f \in \RBV^2(\B_1^d) \st \RTV^2_{\B_1^d}(f) \leq C}.
    \label{eq:RTV2-ball}
\end{equation}
Since $\RTV^2_{\B_1^d}(\dummy)$ is a seminorm on a Banach space, $\F_C$ is convex. The constant function is contained in $\F_C$ since the null space of $\RTV^2_{\B_1^d}(\dummy)$ is the space of affine functions.

Notice that
\[
    \partial \F_C = \F_C - \F_C = 2 \F_C \subset \F_{2C},
\]
so it suffices to upper bound the metric entropy of $\F_{2C}$ to find a $\delta_N$ that satisfies \cref{eq:entropy-inequality}. By noticing that $\norm{\dummy}_N \leq \norm{\dummy}_{L^\infty(\B_1^d)}$, we can use the approximation rate from \cref{thm:approximation} to upper bound (up to logarithmic factors) the metric entropy
\[
    \log \N(t, \F_{2C}, \norm{\dummy}_N) \lessapprox_d \paren{\frac{C}{t}}^{\frac{2d}{d + 3}}
\]
where $\lessapprox$ hides constant and logarithmic factors. The subscript $d$ denotes that the implicit constant depends on $d$. The connection between approximation rates and metric entropy can be viewed as a variant of Carl's inequality~\cite{carl1981entropy} (also see~\cite[Theorem~10]{siegel})

Next,
\begin{align*}
    &\phantom{{}={}} \frac{1}{\sqrt{N}} \int_{\frac{\delta^2}{2\sigma^2}}^\delta \sqrt{\log \N(t, \partial\F, \norm{\dummy}_N)} \dd t \\
    &\leq \frac{1}{\sqrt{N}} \int_0^\delta \sqrt{\log \N(t, \partial\F, \norm{\dummy}_N)} \dd t \\
    &\lessapprox_d \frac{1}{\sqrt{N}} \int_0^\delta \paren{\frac{C}{t}}^{\frac{d}{d + 3}} \dd t \\
    &= \frac{C^{\frac{d}{d+3}}}{\sqrt{N}} \eval*{t^{\frac{3}{d+3}}}_0^\delta \\
    &= C^{\frac{d}{d+3}}\frac{\delta^{\frac{3}{d + 3}}}{\sqrt{N}}.
\end{align*}

From \cref{eq:entropy-inequality}, we want to find $\delta_N = \delta$ that satisfies
\begin{equation}
    C^{\frac{d}{d+3}}\frac{\delta^{\frac{3}{d + 3}}}{\sqrt{N}} \lessapprox_d \frac{\delta^2}{\sigma}.
    \label{eq:satisfies-simplified}
\end{equation}
We have (up to logarithmic factors) that
\[
    \delta_N^2 \asymp_d C^{\frac{2d}{2d + 3}} \paren{\frac{N}{\sigma^2}}^{-\frac{d + 3}{2d + 3}}
\]
satisfies \cref{eq:satisfies-simplified}.
\end{proof}
\section{Proof of \cref{thm:minimax}} \label[app]{app:minimax}
To prove \cref{thm:minimax} we will use the general result of Yang and Barron (see~\cite[Proposition~1]{yang1999information} and~\cite[Chapter~15]{wainwright2019high}) regarding minimax rates over model classes. We summarize this result in the following proposition.

\begin{proposition}[{see~\cite[Proposition~1]{yang1999information} and~\cite[Chapter~15]{wainwright2019high}}] \label{prop:yang-barron}
    Let $\F$ be a model class. Given $f \in \F$, consider the problem of estimating $f$ from the noisy samples
    \[
        y_n = f(\vec{x}_n) + \varepsilon_n, \: n = 1, \ldots, N,
    \]
    where $\curly{\varepsilon_n}_{n=1}^N$ are i.i.d. $\N(0, \sigma^2)$ random variables and $\curly{\vec{x}_n}_{n=1}^N$ are i.i.d. from some probability measure $\P_X$ supported on $\B_1^d$. Then, if functions in $\F$ are uniformly bounded and the metric entropy is of the form
    \[
    \log \N(t, \F, \norm{\dummy}_{L^2(\B_1^d; \P_X)}) \asymp \paren{\frac{1}{t}}^r, \quad r > 0,
    \]
    where $\norm{\dummy}_{L^2(\B_1^d; \P_X)}$ denotes the $L^2$-norm with respect to the measure $\P_X$ on $\B_1^d$, we have the minimax rate
    \[
        \inf_{\hat{f}} \sup_{f \in \F} \, \E\norm*{f - \hat{f}}_{L^2(\B_1^d; \P_X)}^2 \asymp t_N^2,
    \]
    where $t_N^2 = t^2$ satisfies
    \[
        t^2 \asymp \frac{\log \N(t, \F, \norm{\dummy}_{L^2(\B_1^d; \P_X)})}{N}.
    \]
\end{proposition}

We will use the result of \cref{prop:yang-barron} to derive the minimax rate for the model class
\begin{equation}
    \mathscr{G}_C \coloneqq \curly{f \in \V^2(\B_1^d) \st \norm{f}_{\V^2(\B_1^d)} \leq C},
    \label{eq:variation-model-class}
\end{equation}
where $\V^2(\B_1^d)$ is the variation space defined in \cref{sec:other-spaces}. We will then use this minimax rate to derive a minimax lower bound for the model class in \cref{eq:model-class}.

\begin{lemma} \label{lemma:variation-minimax}
    Consider the problem of estimating $f \in \mathscr{G}_C$ (defined in \cref{eq:variation-model-class}) from the noisy samples
    \[
        y_n = f(\vec{x}_n) + \varepsilon_n, \: n = 1, \ldots, N,
    \]
    where $\curly{\varepsilon_n}_{n=1}^N$ are i.i.d. $\N(0, \sigma^2)$ random variables and $\curly{\vec{x}_n}_{n=1}^N$ are i.i.d. uniform random variables on $\B_1^d$. The minimax rate for this model class is
    \[
        \inf_{\hat{f}} \sup_{f \in \mathscr{G}_C} \, \E\norm*{f - \hat{f}}_{L^2(\B_1^d; \P_X)}^2 \asymp_d N^{-\frac{d + 3}{2d + 3}},
    \]
    where the $L^2(\B_1^d; \P_X)$-norm is the $L^2$-norm with respect to the uniform probability measure measure on $\B_1^d$.
\end{lemma}
\begin{proof}
    We are interested in applying \cref{prop:yang-barron} with $\P_X$ being the uniform probability measure on $\B_1^d$. Since the Lebesgue measure is just a constant scaling of the uniform measure (where the constant is the volume of $\B_1^d$), it suffices to know the metric entropy with respect to the $L^2(\B_1^d)$-norm. The model class in \cref{eq:variation-model-class} was extensively studied in~\cite{siegel} and it is known that
    \[
        \log \N(t, \mathscr{G}_C, \norm{\dummy}_{L^2(\B_1^d)}) \asymp_d \paren{\frac{1}{t}}^{\frac{2d}{d+3}}.
    \]
    We refer the reader to~\cite[Theorem~4~and~Equation~(68)]{siegel} for the upper bound and~\cite[Theorem~8]{siegel} for the lower bound. We also remark that the model class $\mathscr{G}_C$ is uniformly bounded since the functions in $\V^2(\B_1^d)$ can be written as a superposition of $L^\infty(\B_1^d)$-bounded atoms. With the metric entropy in the above display, we immediately have the minimax rate in the lemma statement by applying \cref{prop:yang-barron}.
\end{proof}
    
We will now use \cref{lemma:variation-minimax} to derive a minimax lower bound for the model class in \cref{eq:model-class}.

\begin{proof}[Proof~of~\cref{thm:minimax}]
    It suffices to show that $\mathscr{G}_C \subset \F_C$, where $\F_C$ is defined in \cref{eq:RTV2-ball}. Given $f \in \V^2(\B_1^d)$ (or in $\RBV^2(\B_1^d)$, since they are the same space by \cref{thm:RBV-variation-space}), we can find an integral representation as in \cref{eq:integral-representation-variation} such that 
    \[
        \norm{f}_{\V^2(\B_1^d)} = \norm{\mu}_{\M(\Sph^{d-1} \times [-2, 2])}.
    \]
    Next, if we let $\nu \coloneqq \eval{\mu}_{\Sph^{d-1} \times [-1, 1]}$, we can write $f$ as an integral representation as in \cref{rem:identification} such that
    \[
        \RTV^2_{\B_1^d}(f) \leq \norm{\nu}_{\M(\Sph^{d-1} \times [-1, 1])}.
    \]
    The previous two displays imply $\RTV^2_{\B_1^d}(f) \leq \norm{f}_{\V^2(\B_1^d)}$. Therefore, $\mathscr{G}_C \subset \F_C$.
\end{proof}

\section{Proof of \cref{thm:linear-minimax}} \label[appendix]{app:linear-minimax}
To prove \cref{thm:linear-minimax}, we will require several results from ridgelet analysis. It was shown in~\cite[Theorem~7]{candes-phd} that we have the continuous embedding
\[
    R^{(d+3)/2}_{1, 1}(\B_1^d) \subset \V^2(\B_1^d)
\]
where we recall that $\V^2(\B_1^d)$ is the variation space for shallow ReLU networks, and $R^s_{p, q}(\B_1^d)$ denotes the \emph{ridgelet space} of Cand{\`e}s~\cite{candes-phd}. Ridgelet spaces were proposed as a generalization of Besov spaces, and in the univariate case, the ridgelet space $R^s_{p, q}(\B_1^d)$ coincides with the Besov space $B^s_{p,q}[-1, 1]$.

Next, recall that we showed in the proof of \cref{thm:minimax} that $\mathscr{G}_C \subset \F_C$, where $\mathscr{G}_C$ and $\F_C$ are the model classes defined in \cref{eq:variation-model-class} and \cref{eq:model-class}, respectively. Combining this fact with the above display, we see that to prove \cref{thm:linear-minimax}, it suffices to show the linear minimax lower bound for the model class
\[
    \mathscr{H}_C \coloneqq \curly{f \in R^{(d+3)/2}_{1, 1}(\B_1^d) \st \norm{f}_{R^{(d+3)/2}_{1, 1}(\B_1^d)} \leq C}.
\]
We will make use of the following generic result.
\begin{proposition}[{see~\cite[Proof~of~Theorem~4.1]{candes2003ridgelets}}]
    Let $\F \subset L^2(\B_1^d)$ be a convex model class and consider the problem of estimating $f \in \F$ from the continuous white noise model
    \[
        \dd Y_\varepsilon(\vec{x}) = f(\vec{x}) \dd\vec{x} + \varepsilon \dd W(\vec{x}), \quad \vec{x} \in \B_1^d,
    \]
    where $\varepsilon$ is the noise level and $\dd W(\vec{x})$ is a standard $d$-dimensional Wiener process.
    Furthermore, suppose that for any $\delta > 0$, there exists $\lesssim_d K_\delta$ orthogonal elements $\curly{g_k}_{k=1}^K \subset \F$ such that $\norm{g_k}_{L^2(\B_1^d)} = \delta$, $k =1, \ldots, K$. Then, the linear minimax rate is lower-bounded by
    \[
        \inf_{\hat{f} \text{ linear}} \sup_{f \in \F} \E\norm*{f - \hat{f}}_{L^2(\B_1^d)}^2 \gtrsim_d \delta_\varepsilon^2,
    \]
    where $\delta_\varepsilon = \delta$ solves
    \[
        \delta^2 = \varepsilon^2 K_\delta.
    \]
\end{proposition}

\begin{proposition}[{see~\cite[Theorem~11]{candes-phd} and \cite[Lemmas~A.1, A.2, and A.3]{candes2003ridgelets}}]
    For any integer $j \geq 2$, There exists a set $\curly{g_k}_{k=1}^K$ of orthogonal elements with $K \gtrsim_d 2^{jd}$ contained in
    \[
         \curly{f \in R^s_{1, 1}(\B_1^d) \st \norm{f}_{R^s_{1, 1}(\B_1^d)} \leq C},
    \]
    where $C > 0$ is a constant, such that
    \[
        \norm{g_k}_{L^2(\B_1^d)} = 2^{j(s - d/2)}, \: k = 1, \ldots, K.
    \]
\end{proposition}


If we choose $\delta = 2^{j(s - d/2)}$, we see that $K \gtrsim_d \delta^{-2d/(2s - d)}$ and so the linear minimax lower bound is $\delta_\varepsilon^2$, where $\delta_\varepsilon = \delta$ solves
\[
    \delta^2 = \varepsilon^2  \delta^{-2d/(2s - d)},
\]
i.e,
\[
    \delta_\varepsilon^2 = (\varepsilon^2)^{(2s - d) / 2s}.
\]
With these results, we will now prove \cref{thm:linear-minimax}.

\begin{proof}[Proof of \cref{thm:linear-minimax}]
The linear minimax lower bound for the model class $\mathscr{H}_C$ corresponds to the case when $s = (d+3)/2$ and so the linear minimax lower bound for this model class (in the continuous white noise setting) will be
\[
    (\varepsilon^2)^{3/(d+3)}
\]
By a standard sampling argument\footnote{See~\cite{brown1996asymptotic} where this argument was first rigorously formalized in the univariate case, and see~\cite{reiss2008asymptotic} where this idea was rigorously formalized in the multivariate case, which applies to our setting.}, we have that the continuous white noise model is asymptotically equivalent to the estimation problem with discrete samples drawn uniformly on $\B_1^d$, where $\varepsilon = \sigma / \sqrt{N}$, for sufficiently large $N$, so we get the linear minimax lower bound of
\[
    \paren{\frac{N}{\sigma^2}}^{-\frac{3}{d+3}}.
\]
\end{proof}


\bibliographystyle{IEEEtranS}
\bibliography{ref}

\begin{IEEEbiographynophoto}{Rahul Parhi}
received the B.S. degree in mathematics and the B.S. degree in computer science from the University of Minnesota--Twin Cities in 2018, and received the M.S. and
Ph.D. degrees in electrical engineering from the University of Wisconsin--Madison in 2019 and 2022, respectively. During his Ph.D., he was supported by an NSF graduate research fellowship. He is currently a postdoctoral researcher with the Biomedical Imaging Group at the \'Ecole Polytechnique F\'ed\'erale de Lausanne in Switzerland. He is primarily interested in applications of functional and harmonic analysis to problems in signal processing and data science. He is a member of the IEEE.
\end{IEEEbiographynophoto}

\begin{IEEEbiographynophoto}{Robert D. Nowak}
received the Ph.D. degree in electrical engineering from the University of Wisconsin-Madison in 1995.  He was a Postdoctoral Fellow at Rice University from 1995-1996, an Assistant Professor at Michigan State University from 1996-1999, and held Assistant and Associate Professor positions at Rice University from 1999-2003.  Since 2003, Nowak has been with the University of Wisconsin-Madison, where he now holds the Keith and Jane Morgan Nosbusch Professorship in Electrical and Computer Engineering. His research focuses on signal processing, machine learning, optimization, and statistics. His work on sparse signal recovery and compressed sensing has received several awards, including the 2014 IEEE W.R.G. Baker Award.  Nowak has held visiting positions at INRIA, Sophia-Antipolis in 2001, and Trinity College, Cambridge in 2010. He has served as an Associate Editor for the IEEE Transactions on Image Processing and the ACM Transactions on Sensor Networks, and as the Secretary of the SIAM Activity Group on Imaging Science. He was General Chair for the 2007 IEEE Statistical Signal Processing workshop and Technical Program Chair for the 2003 IEEE Statistical Signal Processing Workshop, the 2004 IEEE/ACM International Symposium on Information Processing in Sensor Networks, and the inaugural IEEE GlobalSIP Conference in 2013.  He is presently a Section Editor for the SIAM Journal on Mathematics of Data Science and a Senior Editor for the IEEE Journal on Selected Areas in Information Theory. Nowak is a Fellow of the IEEE.
\end{IEEEbiographynophoto}

\end{document}